\documentclass[master,english,final]{kaist-ucs}

\usepackage{amsmath,amssymb,amsfonts,amsthm}
\usepackage{algorithmic}
\usepackage{graphicx}
\usepackage{caption}
\usepackage{subcaption}
\usepackage{textcomp}
\usepackage{xcolor}
\usepackage{csquotes}

\usepackage{dsfont}
\usepackage{mathtools}
\usepackage{multirow}
\usepackage{color}
\usepackage{anyfontsize}
\usepackage{makecell}
\usepackage{longtable}
\usepackage{hyperref} 
\usepackage{url}
\usepackage{wrapfig,lipsum,booktabs}
\usepackage{mdframed}
\usepackage[linesnumbered,ruled,vlined]{algorithm2e}

\usepackage[utf8]{inputenc}
\usepackage[english]{babel}
\newtheorem{theorem}{Theorem}[section]

\newtheorem{lemma}[theorem]{Lemma}

\usepackage{threeparttable}
\usepackage{graphicx}
\usepackage[referable]{threeparttablex}
\usepackage[draft]{todonotes} 
\usepackage{accents}

\newcommand{\comment}[1]

\newmdenv[
  topline=false,
  bottomline=false,
  skipabove=\topsep,
  skipbelow=\topsep
]{siderules}

\usepackage{framed}
{\begin{leftbar}\begin{quotation}}%
{\end{quotation}\end{leftbar}}

\usepackage{caption} 
\title[korean] {연속 및 이산적 우선 순위를 갖는 VAE의 변이 상호 정보 최대화 프레임 워크}
\title[english]{Variational Mutual Information Maximization Framework for VAE Latent Codes  with Continuous and Discrete Priors}

\author[korean] {안드리}{세르 데가}
\author[korean2] {안드리}{세르 데가}
\author[chinese]{}{}
\author[english]{Serdega}{Andriy}

\advisor[major]{김대식}{Dae-Shik Kim}{signed}
\advisor[major2]{김대식}{Dae-Shik Kim}{signed}    
\advisorinfo{Professor of Electrical Engineering} 
\department{EE}{engineering}{a}

\studentid{20174574}

\referee[1]{Dae-Shik Kim}
\referee[2]{Sae-Young Chung}
\referee[3]{Hyunjoo Jenny Lee}

\approvaldate{2019}{06}{19}

\refereedate{2019}{06}{19}

\gradyear{2019}

\begin{document}
\newcommand{\KL}{\mathit{KL}}
\DeclarePairedDelimiterX{\KLdel}[2]{(}{)}{%
  #1\;\delimsize\|\;#2%
}

   \thesisinfo
    \begin{summary}      
    해석 가능하고 명확한 데이터 표현을 학습하는 것은 기계 학습 연구의 핵심 주제다. Variational Autoencoder(VAE)는 복잡한 데이터의 잠재 변수 모델을 학습할 수 있는 확장 가능한 모델이다. 이는 쉽게 최적화 할 수 있는 명확한 목적 함수를 사용한다. 하지만 이 목적 함수는 학습하는 표현의 품질을 명시적으로 측정하지 않으며, 이는 그 품질을 저하시킬 수 있다. 이 문제를 해결하기 위해 VAE에 대한 가변적 잠재 정보 최대화 프레임워크를 제안한다. 다른 방법들과 달리, 이 방법은 잠재 변수와 관측치 사이의 상호 정보에 대한 하한을 최대화하는 명시적 목적 함수를 제공한다. 이 목적 함수는 VAE가 잠재 변수를 무시하지 않도록 강제하는 정규 표현식의 역할을 하며, 관측치 관점에서 가장 유용한 정보를 주는 부분을 선택하도록 한다. 또한 제안된 프레임 워크는 고정 VAE 모델에 대한 잠재 코드와 관측치 간의 상호 정보를 평가할 수 있는 방법을 제공한다. 우리는 가우시안 및 결합 가우시안 및 이산 잠재 변수를 사용하여 VAE에 대한 실험을 수행했다. 이 결과는 제안된 접근법이 latent codes와 관측치 사이의 관계를 강화하고 학습된 표현을 향상시켜준다는 것을 보여준다.
    \end{summary}
   
    \begin{Korkeyword}
    변분 오토인코더, 변분 추론, 정보 이론, 유도된 잠재 변수 모델, 인공 신경망, 지속적 잠재 변수, 이산 잠재 변수, 가변적 잠재 상호 정보 최대화, 표현 학습
    \end{Korkeyword}

    \begin{abstract}
    Learning interpretable and disentangled representations of data is a key topic in machine learning research. Variational Autoencoder (VAE) is a scalable method for learning directed latent variable models of complex data. It employs a clear and interpretable objective that can be easily optimized. However, this objective does not provide an explicit measure for the quality of latent variable representations which may result in their poor quality.  We propose Variational Mutual Information Maximization Framework for VAE to address this issue. In comparison to other methods, it provides an explicit objective that maximizes lower bound on mutual information between latent codes and observations. The objective acts as a regularizer that forces VAE to not ignore the latent variable and allows one to select particular components of it to be most informative with respect to the observations. On top of that, the proposed framework provides a way to evaluate mutual information between latent codes and observations for a fixed VAE model. We have conducted our experiments on VAE models with Gaussian and joint Gaussian and discrete latent variables. Our results illustrate that the proposed approach strengthens relationships between latent codes and observations and improves learned representations.
    
    \noindent

    \end{abstract} 
     
    \begin{Engkeyword}
     Variational Autoencoder, Variational Inference, Information Theory, directed latent variable model, neural networks, continuous latent variable, discrete latent variable variable, variational mutual information maximization, representation learning.
    \end{Engkeyword}

    \addtocounter{pagemarker}{1}                 
    \newpage

    \tableofcontents

    \listoftables

    \listoffigures



\chapter{Introduction}

Finding a proper data representation can be a crucial part of a given machine learning approach. In many cases, when there is a need for inferring property from a sample, it is the main purpose of the method. For instance, classification task aims to discover such data representation that have a useful high-level interpretation for a human such as class. Unsupervised learning aims to find patterns in unlabeled data that can somehow help to describe it from a human viewpoint and/or perform a relevant task. Recent deep neural network approaches tackle this problem from a perspective of representation learning, where the goal is to learn a representation that captures some semantic properties of data. If learned representations of salient properties are interpretable and disentangled, it would improve generalization and make the downstream tasks robust and easier~\cite{lake2016building}. 

Over the last decade, generative models have become popular in unsupervised learning research. The intuition is that by generative modeling it may be possible to discover latent representations and their relation to observations. The two most popular training frameworks for such models are Generative Adversarial Networks (GAN) \cite{goodfellow2014generative} and Variational Autoencoders \cite{kingma2013auto, rezende2014stochastic}. The latter is a powerful method for unsupervised learning of directed probabilistic latent variable models. Training a model within Variational Autoencoder (VAE) framework allows performing both tasks of inference and generation. This models are trained by maximizing evidence lower bound (ELBO) which is a clear objective and results in stable training in comparison to GAN. However, the latent variable in ELBO is marginalized and the resulting objective depends only on $p_{\boldsymbol{\theta}}(\mathbf{x})$. Therefore, it does not assess the ability of the model to do inference and the quality of latent code \cite{huszar2017maximum,alemi2018fixing}. Thus, having a high ELBO does not necessarily mean that useful latent representations were learned. Moreover, powerful enough decoder can ignore the conditioning on latent code resulting in $p_{\boldsymbol{\theta}}(\mathbf{x}|\mathbf{z})=p_{\boldsymbol{\theta}}(\mathbf{x})$ \cite{bowman2016generating,chen2016variational}. On top of that, even when a strong assumption about underlying generative process is incorporated into the latent variable, the model trained with ELBO objective may not assign any interpretable representation to it as we will show in our experimental results.

Despite the lack of influence on latent codes in ELBO, it is possible for VAE to learn interpretable and disentangled representations as it was shown by the auhtors. In the setting of VAE, approximation to the marginal log-likelihood is maximized that may recover true generative process with interpretable and useful latent representations. This approach may lead to feasible results when there are strong constraints on joint distribution. It also relies on expressiveness of the selected model and parameter initialization. Learned latent representations can be also improved by employing weighting and constraining coefficients for VAE objective terms or varying them~\cite{higgins2017beta,kim2018disentangling,chen2018isolating}. However, it is not clear which part of the code will actually learn and control the main interpretable underlying variation factors. To discover it, one should make strong assumption about underlying generative process before training and, after training, manually tweak each component of the code and check how it is related to produced data samples.

The key idea of our approach is to maximize mutual information (MI) between samples from posterior distribution (represented by encoder) and observations. Unfortunately, exact MI computing is hard and may be intractable. To overcome this, our framework employs Variational Information Maximization \cite{barber2003algorithm} to obtain lower bound on true MI. This technique relies on approximation by auxiliary distribution and we represent it by the additional inference network. The obtained lower bound on MI is used as the regularizer to the original VAE objective to force the latent representations to have strong relationship with observations and prevent the model from ignoring them. 

We have conducted our experiments on VAE models with such latent distributions: Gaussian, and joint Gaussian and discrete. We compare qualitatively and quantitatively the models trained using pure ELBO objective and with introduced MI regularizer.

\chapter{Theoretical Background}
\section{Probabilistic View of The World \label{sec:prob_view}}
Probability theory is a mathematical framework that allows one to quantify uncertainty of some phenomena or event as well as provides foundation for analysis of hypothesis and uncertain statements. Probability theory is a main tool for quantifying our uncertainty about the world and it allows us to still make decisions when being totally certain is impossible. The world around is stochastic and unpredictable from a human viewpoint as well as from Quantum mechanics perspective. It would be fair to state that, in our world we have to reason and make decisions while uncertainty is always present. Back in time, Pierre-Simon Laplace stated: 
\begin{displayquote}
We may regard the present state of the universe as the effect of its past and the cause of its future. An intellect which at a certain moment would know all forces that set nature in motion, and all positions of all items of which nature is composed, if this intellect were also vast enough to submit these data to analysis, it would embrace in a single formula the movements of the greatest bodies of the universe and those of the tiniest atom; for such an intellect nothing would be uncertain and the future just like the past would be present before its eyes.
\end{displayquote}
By this, Laplace argues that given the whole knowledge, the demon could correctly predict anything. Thus, from general standpoint, there is nothing stochastic in our world and just lack of observations and knowledge is the cause uncertainty. It had given a rise of causal scientific determinism. However, with the raise of Quantum mechanics it was proved that determinism (on quantum interaction level) is incompatible with with the Copenhagen interpretation, which stipulates indeterminacy. With regard to the aforementioned statements and according to \cite{Goodfellow-et-al-2016} we can argue that there are exist two sources of uncertainty:
\begin{enumerate}
  \item Unavoidable stochasticity present in the observed real-world phenomena that are stochastic by their nature. In Quantum mechanics, the dynamics of particles can be modelled only in probabilistic way.
  \item Lack of knowledge and observations. The systems that are deterministic by nature may be stochastic when there are no possibility to gain knowledge about all underlying factors that affect it.
\end{enumerate}
 
The first source can be also incorporated into objectivist view on randomness. It tracts randomness as a fundamental property of the world. On the other hand, second source align with subjectivist perspective. It tracts randomness of event as a lack of knowledge about it. It results in uncertainty degree that is represented by probabilities. In machine learning and artificial intelligence, the subjectivist view is the only reasonable point to tackle particular problems. In this setting, probability theory laws provide an intuition for design of machine leaning algorithms and systems. From the general viewpoint, the main goal of conventional machine learning model is to perform inference. It can be formulated as an estimation of one quantity using knowledge about the other one while they are statistically related. On top of that, probability theory and statistics provide tools for theoretical analysis of various machine learning systems.
\section{Probabilistic Modelling\label{sec:prob_modeling}}
\subsection{Intuition}
The very general form to define a mathematical model for inference for the real-world problem is the equation of the form
\begin{equation} \label{eq:simple_model}
    \mathbf{y}=\boldsymbol{\mathbf{W}}\mathbf{x},
\end{equation}
where $\mathbf{y}$ is some variable that we want to infer with known value of variable $\mathbf{x}$ that is somehow related to $\mathbf{y}$, and $\boldsymbol{\mathbf{W}}$ defines some arbitrary transformation. As we mentioned in previous section, different phenomena in the real world are rarely deterministic and involve significant stochasticity. Thus, it is natural to fully model a particular real world phenomena in a from of probability distribution
\begin{equation} \label{eq:world_probab}
    p(\mathbf{x},\mathbf{y}).
\end{equation}

Probability can be seen as the fuzzy form of logic that is applicable in cases with uncertainty. Logic is a set of rules for determining whether an argument true or false having set of prior arguments. The connection is that probability theory formalizes how to determine the likelihood of an argument  given the likelihood of prior arguments.
\subsection{Probabilistic Models}
Probabilistic models are used to describe systems with the use of probability distributions and random variables. Machine learning employs probabilistic models to recover true data generating distribution having set of samples $\mathcal{D}=\{\mathbf{x}^{(1)}, \mathbf{x}^{(2)}, \cdots, \mathbf{x}^{(m)}\}$ (where $m\geq1$) drawn from true distribution $p_{data}(\mathbf{x})$. We wish to approximate the true distribution using a model $p_{\boldsymbol{\theta}}$ that represents some parametric family of distributions with parameters $\boldsymbol{\theta}$, such that $p_{\boldsymbol{\theta^*}}(\mathbf{x}) \approx p_{data}(\mathbf{x})$. The most common approach to tackle this problem in modern machine learning and deep learning is Maximum Likelihood Estimation which results in optimal point estimate of $\boldsymbol{\theta}$ using estimator
\begin{equation} \label{eq:general_mle}
    \boldsymbol{\theta^*}=\mathop{\mathrm{arg\,max}}_{\boldsymbol{\theta}} p_{\boldsymbol{\theta}}(\mathcal{D}).
\end{equation}
If the samples $\mathcal{D}$ from the true data distribution are i.i.d and using property of $\log$, equation ~\ref{eq:general_mle} can be rewritten as 
\begin{equation} \label{eq:final_mle}
    \boldsymbol{\theta^*}=\mathop{\mathrm{arg\,max}}_{\boldsymbol{\theta}}\sum_{i = 1}^{m}[\log p_{\boldsymbol{\theta}}(\mathbf{x^{(i)}})]=\mathop{\mathrm{arg\,max}}_{\boldsymbol{\theta}}\mathbb{E}_{\mathbf{x}\sim\mathcal{D}}[\log p_{\boldsymbol{\theta}}(\mathbf{x})].
\end{equation}
Actually, maximum likelihood estimation can be seen as minimizing KL divergence between empirical distribution $\hat{p}_{data}$ (represented by $\mathcal{D}$) of true data distribution $p_{data}(\mathbf{x})$ and the distribution represented by the model from chosen parametric family $p_{\boldsymbol{\theta}}$. KL divergence is a measure of how one probability distribution is different from another probability distribution. It has a perfect sense in our setting and then KL divergence can be defined as
\begin{equation} \label{eq:kl_mle}
    D_{\KL}\KLdel{\hat{p}_{data}}{p_{\boldsymbol{\theta}}}=\mathbb{E}_{\mathbf{x}\sim\hat{p}_{data}}[\log\hat{p}_{data}(\mathbf{x}) - \log p_{\boldsymbol{\theta}}(\mathbf{x})].
\end{equation}
Since the first term is constant, to minimize the KL divergence we need to minimize the negative second term which is the same as maximization setting in equation~\ref{eq:final_mle}.

In the modern machine learning, this parametric distribution families are often represented by artificial neural networks (NN). I has been proven that this kind of models even with a single hidden layer with a finite number of units is capable of approximating any continuous function on compact subsets of real numbers \cite{cybenko1989approximation}. Thus, the problem formulated in equation~\ref{eq:final_mle} can be seen as an optimization problem with respect to the parameters $\boldsymbol{\theta}$ which are the parameters of some NN. 
Optimization for NN when the are no closed form solution is usually done with gradient descent using gradients obtained by backpropagation algorithm \cite{rumelhart1985learning}. 

In practice, for big datasets and models the optimization is done using stochastic gradient descent (SGD) by randomly drawing the minibatches of data $\mathcal{M} \subset \mathcal{D}$ and obtaining unbiased gradient estimator for likelihood maximization:
\begin{align}
\frac{1}{|\mathcal{D}|} \nabla_{\boldsymbol{\theta}} \textrm{ log }p_{\boldsymbol{\theta}}(\mathcal{D}) = \mathbb{E}_{\mathcal{M} \sim \mathcal{D}} \bigg[ \frac{1}{|\mathcal{M}|} \underset{\mathbf{x}\in \mathcal{M}}{\sum} \nabla_{\boldsymbol{\theta}} \textrm{ log }p_{\boldsymbol{\theta}}(\mathbf{x}) \bigg].
\label{eq:ml_grad}
\end{align}
\subsection{Discriminative and Generative Models}
\textbf{Discriminative models}. A discriminative model is a model that defines conditional probability distribution $p(\mathbf{y}|\mathbf{x})$ for the target variable $\mathbf{y}$, given a known observation $\mathbf{x}$. When this kind models is applied to prediction tasks, we can formulate this as problem of unknown variable $\mathbf{y}$ estimation using observation $\mathbf{x}$. In this setting, maximum likelihood estimation can be extended to estimate parameters $\boldsymbol{\theta}$ of some distribution $p_{\boldsymbol{\theta}}(\mathbf{y}|\mathbf{x})$ that we want to use for an approximation of real conditional distribution  $p_{true}(\mathbf{y}|\mathbf{x})$ to do inference or predictions. In that setting, given a dataset comprised of data points and respective predictions (labels) $\mathcal{D}=\{(\mathbf{x}^{(1)},\mathbf{y}^{(1)}), (\mathbf{x}^{(2)},\mathbf{y}^{(2)}), \cdots, (\mathbf{x}^{(m)},\mathbf{y}^{(m)})\}$ (where $m\geq1$), the conditional maximum likelihood estimation has form
\begin{equation} \label{eq:cond_mle}
    \boldsymbol{\theta^*}=\mathop{\mathrm{arg\,max}}_{\boldsymbol{\theta}}\sum_{i = 1}^{m}[\log p_{\boldsymbol{\theta}}(\mathbf{y^{(i)}}|\mathbf{x^{(i)}})]=\mathop{\mathrm{arg\,max}}_{\boldsymbol{\theta}}\mathbb{E}_{(\mathbf{x},\mathbf{y})\sim\mathcal{D}}[\log p_{\boldsymbol{\theta}}(\mathbf{y}|\mathbf{x})].
\end{equation}

\textbf{Generative models}. This kind of models aims to find all factors that explain target phenomena or data and define distribution $p(\mathbf{x},\mathbf{y})$. From the previous section viewpoint, we can see it as a probabilistic model that simultaneously models distribution of observations and predictions. This allows one to obtain discriminative distribution $p(\mathbf{y}|\mathbf{x})$ as well as generative one $p(\mathbf{x}|\mathbf{y})$ and individual distributions of $p(\mathbf{x})$ and $p(\mathbf{y})$ by marginalization and Bayes rule.
\section{Directed Probabilistic Models\label{sec:graph_models}}
\subsection{Formal definition}
Machine learning and Deep Learning approaches face probability distributions over a various number of random variables. These distributions may involve direct interactions between relatively few variables in comparison to their total number. Using a single function to describe the entire joint probability distribution over all random variables may be inefficient and does not provide a way to incorporate knowledge about direct dependencies between this variables into a model.

Directed graphical models or Bayesian networks is a family of probability distributions that provide compact parametrization that can be naturally described using a directed graph. Bayesian networks represent probability distributions that can be formed via products of smaller conditional probability distributions (one for each variable). By having probability of this form, we are introducing assumptions that certain variables are independent and others are dependent into our model. The parametrization idea is next, having joint probability of all present random variables, we can decompose it using chain rule as
\begin{equation}\label{eq:chain}
    p(\mathbf{x}_{1}, \mathbf{x}_{2}, \ldots, \mathbf{x}_{n})=p(\mathbf{x}_{1})p(\mathbf{x}_{2}|\mathbf{x}_{1})\cdots p(\mathbf{x}_{n}|\mathbf{x}_{n-1}, \ldots, \mathbf{x}_{2}, \mathbf{x}_{1})
\end{equation}

In a Bayesian network model representation, the components of factorized right-hand side of the Eq.\ref{eq:chain} is simplified when prior assumption about independence of certain variables are incorporated into the model. Then, each factor for $\mathbf{x}_{i}$ depends only on a particular set of ancestor variables $\mathbf{x}_{A_{i}}$ resulting in 
\begin{equation}\label{eq:factor_simple}
    p(\mathbf{x}_{i}|\mathbf{x}_{i-1}, \ldots, \mathbf{x}_{1})=p(\mathbf{x}_{i}|\mathbf{x}_{A_{i}}).
\end{equation}
In this setting, we can reformulate the initial joint distribution as \begin{equation}\label{eq:chain_bayesian}
    p(\mathbf{x}_{1}, \mathbf{x}_{2}, \ldots, \mathbf{x}_{n})=\prod_{i=1}^n p(\mathbf{x}_{i}|\mathbf{x}_{A_{i}})
\end{equation}

Distributions and models of this form can be naturally expressed as directed acyclic graphs (DAG), where nodes represent random variables $\mathbf{x}_{i}$ and edges represent their relationships. For instance, consider a probabilistic directed model over random variables $\mathbf{a},\mathbf{b},\mathbf{c},\mathbf{d},\mathbf{e}$ which is represented by figure~\ref{fig:basic_dag}. This graphical model enables one directly incorporate relationships between random variables into probabilistic model of the distribution. As you can see,  $\mathbf{a}$ and $\mathbf{c}$ (as well as $\mathbf{a}$ and $\mathbf{b}$) interact directly, but $\mathbf{a}$ and $\mathbf{e}$ interact only indirectly through $\mathbf{c}$ forming markov chain.
\begin{figure}[h]
\centering
\includegraphics[width=0.25\linewidth]{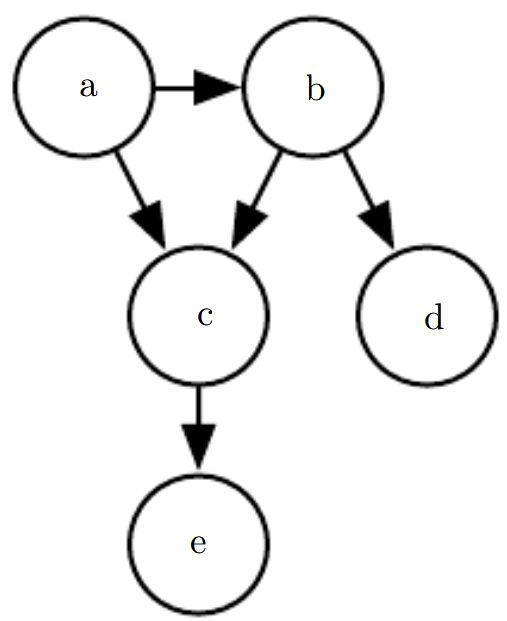}
\caption{A directed graphical model over random variables a,b,c,d, and e (from~\cite{Goodfellow-et-al-2016}).}
\label{fig:basic_dag}
\end{figure}
Such model can be represented by joint probability distribution that is factorized as
\begin{equation}\label{eq:basic_dag_distr}
    p(\mathbf{a},\mathbf{b},\mathbf{c},\mathbf{d},\mathbf{e})=p(\mathbf{a})p(\mathbf{b}|\mathbf{a})p(\mathbf{d}|\mathbf{b})p(\mathbf{e}|\mathbf{c})p(\mathbf{c}|\mathbf{a},\mathbf{b})
\end{equation}

\textbf{Formal definition}. Directed probabilistic model is a directed graph $\mathbf{G}$=($\mathbf{V}$,$\mathbf{E}$) where each random variable $\mathbf{x}_{i}$ corresponds to node \textit{i} $\in$ $\mathbf{V}$ and $\mathbf{E}$ represents relationships between this variables. One conditional probability distribution per node $p(\mathbf{x}_{i}|\mathbf{x}_{A_{i}})$ from factorized joint distribution over all variables. It specifies probability distribution of $\mathbf{x}_{i}$ with respect to its ancestors given values $\mathbf{x}_{A_{i}}$. The advantage of structured probabilistic modelling is that it reduces the modelling cost of probability distributions to explain some particular phenomena. Also, it makes it is easier to perform learning and inference. The main advantage is possibility to avoid modeling particular dependencies with use of prior assumptions. These models represent dependencies as edges. Leaving any particular edge out, the model incorporates an assumption that some random variables do not require direct interaction.

\subsection{Dependencies in Directed Probabilistic Models} Independence assumptions made when modelling Bayesian network are important since they define key properties of the particular model. They represent prior understanding of the phenomena that we are trying to model and this information may provide useful insights for performing inference. There are three key dependence structures for modeling probabilistic models with three random variables (nodes). Let Bayes net represent three random variables $\mathbf{A}$, $\mathbf{B}$, and $\mathbf{C}$. Then these three dependence structures (illustrated in figure~\ref{fig:dependency_graphs}) can be formulated as follows
\begin{itemize}
  \item \textit{Cascade}. In that case directed graphical model represent dependencies as $\mathbf{A}$ $\rightarrow$ $\mathbf{B}$ $\rightarrow$ $\mathbf{C}$. Thus, if variable $\mathbf{B}$ is observed then  $\mathbf{A}$ and $\mathbf{C}$ are conditionally independent. This is so since in that setting, the variable B contains everything that define outcome $\mathbf{C}$.
  \item \textit{Common parent}. When Bayesian model has form of $\mathbf{A}$ $\leftarrow$ $\mathbf{B}$ $\rightarrow$ $\mathbf{C}$ and $\mathbf{B}$ is observed then $\mathbf{A}$ and $\mathbf{C}$ are conditionally independent again. In that case, this is so since $\mathbf{B}$ completely defines outcome $\mathbf{A}$ and $\mathbf{B}$ independently.
  \item \textit{V-structure}. In that case the model has form of $\mathbf{A}$ $\rightarrow$ $\mathbf{C}$ $\leftarrow$ $\mathbf{B}$. It means that knowledge of outcome $\mathbf{C}$ couples $\mathbf{A}$ and $\mathbf{B}$. Otherwise, $\mathbf{A}$ and $\mathbf{B}$ are independent.
\end{itemize}
\begin{figure}[h]
\centering
\includegraphics[width=0.6\linewidth]{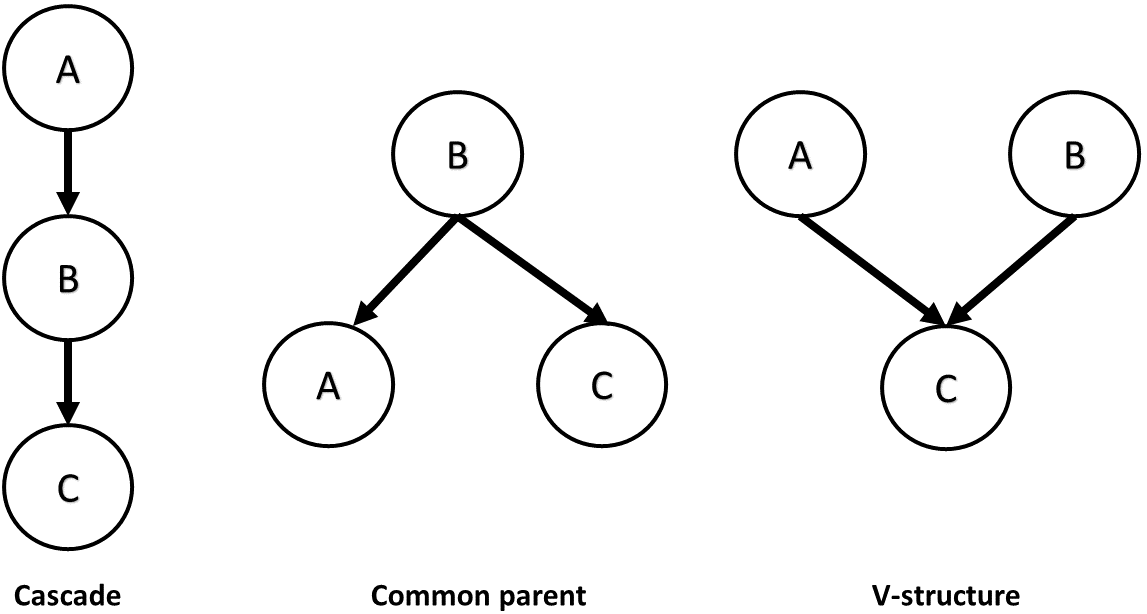}
\caption{Dependency structures in probabilistic models of three random variables}
\label{fig:dependency_graphs}
\end{figure}
\section{Latent Variable Models\label{sec:latent_variable_mods}}
\subsection{General Setting}
Latent variable models (LVM) is powerful approach to probabilistic modelling that deals with a set of observed variables with latent, or hidden, variables. This modelling approach is extremely useful in cases when some data might be naturally unobserved. When a joint distribution over observable and latent variables is defined, the distribution of the observations can be obtained by marginalization of the latent variable. This approach makes it possible to model complex data distributions and express them in terms of more tractable joint distributions. The common example of a latent variable model is the Gaussian mixture distribution where the latent variable is mixture component label. The structure of such probabilistic models can be represented in a graphical manner. For instance, as a directed acyclic graph (DAG), or Bayesian network. 

In general, latent variable model is a model that defines joint probability distribution over two random variables $\mathbf{x}$ and $\mathbf{z}$
\begin{equation}\label{eq:general_lvm}
    p(\mathbf{x},\mathbf{z}),
\end{equation}
where $\mathbf{x}$ are observable variable (represented by data points during training) and $\mathbf{z}$ that are unobserved. Latent	variables may appear naturally, from the	structure of the problem, because	something	wasn’t	measured,	because	of faulty sensors, occlusion, privacy. These models can be represented as directed or undirected. There exist both discriminative and generative LVM and those which combine them as Variational Autoencoder \cite{kingma2013auto, rezende2014stochastic}.
The goal of LVM is to express distribution $p(\mathbf{x})$ of the observable variables using smaller number of latent variables~\cite{jordan1998learning}. It is done by first decomposing the joint distribution $p(\mathbf{x},\mathbf{z})$ into posterior distribution $p(\mathbf{x}|\mathbf{z})$ and prior $p(\mathbf{z})$ resulting in
\begin{equation}\label{eq:lvm_decomposition}
    p(\mathbf{x},\mathbf{z})=p(\mathbf{x}|\mathbf{z})p(\mathbf{z}).
\end{equation}
Then, the distribution $p(\mathbf{x})$ is obtained by marginalization of latent variable
\begin{equation}\label{eq:lvm_decomposition}
    p(\mathbf{x})=\int p(\mathbf{x}|\mathbf{z})p(\mathbf{z}) d\mathbf{z}
\end{equation}
to which can be appllied such estimation techniques as maximum likelihood estimation (eq.~\ref{eq:final_mle}). The integration procedure in general can be intractable, as in case of deep neural networks, except cases when the distributions $p(\mathbf{x}|\mathbf{z})$ and $p(\mathbf{z})$ have specific form and closed form solution is achievable.
\subsection{Gaussian Mixture Models}
Gaussian mixture models (GMMs) is a kind of widely used latent variable models in machine learning. In this setting, each data point is a tuple $(\mathbf{x}_i,\mathbf{z}_i)$ where $\mathbf{x}_i$$\in$$\mathbb{R}^d$ and $\mathbf{z}_i$ is a discrete random variable of K categories. The joint distribution $p(\mathbf{x},\mathbf{z})=p(\mathbf{x}|\mathbf{z})p(\mathbf{z})$ is a directed model (see fig~\ref{fig:gmm_graph}), where $p(\mathbf{z}=k)=\pi_k$ for some vector of class probabilities $\pi\in\Delta_{K-1}$ and then
\begin{equation}\label{eq:mixture_xz}
    p(\mathbf{x}|\mathbf{z}=k)=\mathcal{N}(\mathbf{x};\mu_k,\Sigma_k)
\end{equation}
which is multivariate Gaussian distribution parametrized by $\mu_k$ and $\Sigma_k$. GMM model defines observed data as points that come from K Gaussian clusters with proportions represented by $\pi_1,\ldots,\pi_{K}$. The resulting observed data point probability has form of
\begin{equation}\label{eq:mixture_observ_prob}
    p(\mathbf{x})=\sum_{n=1}^{K}p(\mathbf{x}|\mathbf{z}=n)p(\mathbf{z}=n)=\sum_{n=1}^{K}\pi_n\mathcal{N}(\mathbf{x};\mu_n,\Sigma_n)
\end{equation}
If we wish to sample new data point from that distribution, first we sample a category k and then sample a point from Gaussian distribution with respective parameters $\mu_k$ and $\Sigma_k$.
\begin{figure}[h]
\centering
\includegraphics[width=0.15\linewidth]{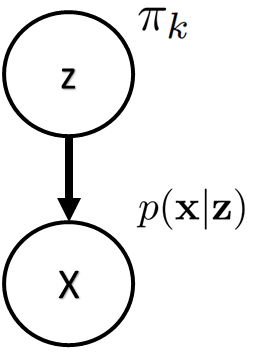}
\caption{Graphical directed model representation of GMM}
\label{fig:gmm_graph}
\end{figure}
\subsection{Fitting Latent Variable Models. EM algorithm}
Expectation Maximization (EM) algorithm is widely used algorithm for fitting directed LVM models which aim is the same as in the case of maximum likelihood estimation - maximize data likelihood (see section 2.4.1)  $p_{\boldsymbol{\theta}}(\mathbf{x})$ for some model $p_{\boldsymbol{\theta}}(\mathbf{x},\mathbf{z})$ of particular parametric family with parameters $\boldsymbol{\theta}$. Having collected observable data points $\mathcal{D}=\{\mathbf{x}^{(1)}, \mathbf{x}^{(2)}, \cdots, \mathbf{x}^{(m)}\}$ (where $m\geq1$), we wish to maximize marginal log-likelihood of the data
\begin{equation} \label{eq:marginal_logli}
   \log[p_{\boldsymbol{\theta}}(\mathcal{D})]=\mathbb{E}_{\mathbf{x}\sim\mathcal{D}}[\log p_{\boldsymbol{\theta}}(\mathbf{x})]=\mathbb{E}_{\mathbf{x}\sim\mathcal{D}}[\mathbb{E}_{\mathbf{z}\sim p_{\boldsymbol{\theta}}(\mathbf{z}|\mathbf{x})}[\log p_{\boldsymbol{\theta}}(\mathbf{x},\mathbf{z})]].
\end{equation}
This optimization objective is more complex than log-likelihood maximization in maximum likelihood estimation setting, even for directed graphical models. It is not possible to derive a closed form solution for the parameters even when the distributions represent some conventional distribution families.

On the other hand, we can employ iterative EM algorithm that may converge to optimal solution. This algorithm relies on assumption that it is possible to obtain posterior distribution $p_{\boldsymbol{\theta}}(\mathbf{z}|\mathbf{x})$. There are cases when it is not possible and there arise variational inference methods that will be covered later. The EM algorithm relies on two steps, first we start with some arbitrary parameters $\boldsymbol{\theta}_0$ and then repeat the following steps until convergence with t=0,1,2...
\begin{itemize}
  \item \textit{E-step}. For each $\mathbf{x}$ $\in$ $\mathcal{D}$, compute posterior $p_{\boldsymbol{\theta_t}}(\mathbf{z}|\mathbf{x})$.
  \item \textit{M-step}. Obtain new weights by
  \begin{equation} \label{eq:mle_marginal_logli}
      \boldsymbol{\theta_{t+1}}=\mathop{\mathrm{arg\,max}}_{\boldsymbol{\theta}}\mathbb{E}_{\mathbf{x}\sim\mathcal{D}}[\mathbb{E}_{\mathbf{z}\sim p_{\boldsymbol{\theta}_{t}}(\mathbf{z}|\mathbf{x})}[\log p_{\boldsymbol{\theta}_{t}}(\mathbf{x},\mathbf{z})]]=\mathop{\mathrm{arg\,max}}_{\boldsymbol{\theta}}\mathbb{E}_{\mathbf{x}\sim\mathcal{D}}[\log p_{\boldsymbol{\theta}_{t}}(\mathbf{x})].
  \end{equation}
\end{itemize}

Since the optimization procedure is non-convex, EM is not guaranteed to reach global optimum solution. However, EM in practical application converges to a feasible local optimum solution. The performance of EM algorithm depends on the initial choice of parameters $\boldsymbol{\theta}_0$. Different initialization may lead to different local optima. Techniques for selecting parameters for particular probabilistic model to which EM is applicable is an active separate direction in machine learning research.

\subsection{Variational Inference}
\textbf{Inference as optimization} In this setting, consider LVM again with observed variables $\mathbf{x}$ and latent variables $\mathbf{z}$. For fitting the particular model we wish to obtain $\log p_{\boldsymbol{\theta}}(\mathbf{x})$ for the observed data. There are cases when it is costly and intractable to marginalize the latent variable $\mathbf{z}$. Instead, it is convenient to maximize variational lower bound or evidence lower bound (ELBO). The other used term for this objective is variational free energy proposed by Richard Feynman in 1972. Evidence lower bound can be defined as
\begin{equation} \label{eq:free_energy}
\mathcal{L}(\mathbf{z},\boldsymbol{\theta},q)=\log p_{\boldsymbol{\theta}}(\mathbf{x}) - D_{\KL}\KLdel{q(\mathbf{z})}{p_{\boldsymbol{\theta}}(\mathbf{z}|\mathbf{x})}
\end{equation}
where q is an arbitrary distribution. We can derive it as follows
\begin{align}
    \textrm{log }p_{\boldsymbol{\theta}}(\mathbf{x}) &= \mathbb{E}_{q(\mathbf{z})} \Big[ \textrm{log }p_{\boldsymbol{\theta}}(\mathbf{x}) \Big] \\
    &= \mathbb{E}_{q(\mathbf{z})} \bigg[ \textrm{log} \left\{ \frac{p_{\boldsymbol{\theta}}(\mathbf{x}, \mathbf{z})}{q(\mathbf{z})} \frac{q( \mathbf{z})}{p_{\boldsymbol{\theta}}(\mathbf{z}| \mathbf{x})} \right\} \bigg] \\
    &= \mathbb{E}_{q(\mathbf{z})} \bigg[ \textrm{log} \left\{ \frac{p_{\boldsymbol{\theta}}(\mathbf{x}, \mathbf{z})}{q(\mathbf{z})}  \right\}\bigg] + 
    \mathbb{E}_{q(\mathbf{z})} \bigg[ \textrm{log} \left\{ \frac{q( \mathbf{z})}{p_{\boldsymbol{\theta}}(\mathbf{z}| \mathbf{x})} \right\} \bigg] \\
    &= \mathcal{L}(\mathbf{z},\boldsymbol{\theta},q) + D_{\KL}\KLdel{q(\mathbf{z})}{p_{\boldsymbol{\theta}}(\mathbf{z}|\mathbf{x})}.
\end{align}
Since the KL divergence term is non-negative, ELBO can be seen as lower bound on true data log likelihood, meaning
\begin{equation} \label{eq:elbo_le}
\log p_{\boldsymbol{\theta}}(\mathbf{x}) \geq
\mathcal{L}(\mathbf{z},\boldsymbol{\theta},q).
\end{equation}
We can interpret it as an approximation to the true data likelihood and it is tight when 
\begin{equation} \label{eq:KL_zero}
D_{\KL}\KLdel{q(\mathbf{z})}{p_{\boldsymbol{\theta}}(\mathbf{z}|\mathbf{x})}=0 \iff q(\mathbf{z})=p_{\boldsymbol{\theta}}(\mathbf{z}|\mathbf{x}),
\end{equation}
and then
\begin{equation} \label{eq:elbo_eq}
\log p_{\boldsymbol{\theta}}(\mathbf{x}) =
\mathcal{L}(\mathbf{z},\boldsymbol{\theta},q).
\end{equation}

With respect to aforementioned theoretic insights, it is valid to interpret inference as the procedure optimizing $\mathcal{L}(\mathbf{z},\boldsymbol{\theta},q)$ with respect to q. Exact inference can match $\mathcal{L}(\mathbf{z},\boldsymbol{\theta},q)$ to $\log p_{\boldsymbol{\theta}}(\mathbf{x})$ if q represents a set of distributions that includes $p_{\boldsymbol{\theta}}(\mathbf{z}|\mathbf{x})$. Moreover, in some cases when this distributions have particular form, this optimization can be done in a closed form.
\subsection{EM as ELBO Maximization}
\textbf{Variational lower bound}. Function $g(\xi,\mathbf{y})$ is called variational lower bound for $f(\mathbf{y})$ iff for all $\xi$ and for all $\mathbf{y}$ it follows $f(\mathbf{y})$ $\geq$ $g(\xi,\mathbf{y})$ and for any $\mathbf{y}_{0}$ there exists $\xi(\mathbf{y}_{0})$ such that $f(\mathbf{y}_{0})$ = $g(\xi(\mathbf{y}_{0}),\mathbf{y}_{0})$.
If it is possible to find such variational lower bound then, instead of solving
\begin{equation} \label{eq:var_lb_max}
    \mathbf{y}^{*}=\mathop{\mathrm{arg\,max}}_{\mathbf{y}} f(\mathbf{y}),
\end{equation}
it is possible to perform block-coordinate updates of $g(\xi(\mathbf{y}),\mathbf{y})$, meaning
\begin{equation} \label{eq:block_coord_upd}
    \mathbf{y}_{t}=\mathop{\mathrm{arg\,max}}_{\mathbf{y}} g(\xi_{t-1},\mathbf{y}) \ ,\  \xi_{t}=\xi(\mathbf{y}_{t})=\mathop{\mathrm{arg\,max}}_{\xi} g(\xi,\mathbf{y}_{t}) 
\end{equation}

In the setting of latent variable models, we wish to maximize marginal data log-likelihood $\log p_{\boldsymbol{\theta}}(\mathbf{x})$. In the Section 2.4.3, we described a common method for fitting this kind of models - expectation maximization algorithm. This algorithm can be also described from the perspective of maximization of variational lower bound on marginal data log-likelihood (ELBO). From (2.23) we know that
\begin{equation} \label{eq:em_elbo_1}
\log p_{\boldsymbol{\theta}}(\mathbf{x}) =
\mathcal{L}(\mathbf{z},\boldsymbol{\theta},q) + D_{\KL}\KLdel{q(\mathbf{z})}{p_{\boldsymbol{\theta}}(\mathbf{z}|\mathbf{x})} \ and \ \log p_{\boldsymbol{\theta}}(\mathbf{x}) \geq
\mathcal{L}(\mathbf{z},\boldsymbol{\theta},q).
\end{equation}
Then, to maximize $\log p_{\boldsymbol{\theta}}(\mathbf{x})$ we can maximize ELBO when it is tight. Thus, we can reformulate EM algorithm as follows: starting with $\boldsymbol{\theta}_0$ and then repeating the following steps until convergence with t=0,1,2...
\begin{itemize}
  \item \textit{E-step}. For each $\mathbf{x}$ $\in$ $\mathcal{D}$ compute posterior which is equivalent to making ELBO tight by finding q that maximizes it:
  \begin{equation} \label{eq:em_elbo_e_step}
      q(\mathbf{z})=\mathop{\mathrm{arg\,max}}_{q}\mathcal{L}(\mathbf{z},\boldsymbol{\theta},q)=\mathop{\mathrm{arg\,min}}_{q}D_{\KL}\KLdel{q(\mathbf{z})}{p_{\boldsymbol{\theta}_{t}}(\mathbf{z}|\mathbf{x})}=p_{\boldsymbol{\theta}_{t}}(\mathbf{z}|\mathbf{x})
  \end{equation}
  \item \textit{M-step}. Having ELBO tight to true data log-likelihood, we can find the parameters $\boldsymbol{\theta_{t+1}}$ that maximize it
  \begin{equation} \label{eq:em_elbo_m_step}
      \boldsymbol{\theta_{t+1}}=\mathop{\mathrm{arg\,max}}_{\boldsymbol{\theta}}\mathcal{L}(\mathbf{z},\boldsymbol{\theta},q)=\mathop{\mathrm{arg\,max}}_{\boldsymbol{\theta}}\log p_{\boldsymbol{\theta}_{t}}(\mathbf{x})=\mathop{\mathrm{arg\,max}}_{\boldsymbol{\theta}}\mathbb{E}_{\mathbf{z}}[\log p_{\boldsymbol{\theta}_{t}}(\mathbf{x},\mathbf{z})].
  \end{equation}
\end{itemize}
This process monotonically increases the lower bound and converges to stationary point of $\log p_{\boldsymbol{\theta}^*}(\mathbf{x})$. If the true posterior $p_{\boldsymbol{\theta}}(\mathbf{z}|\mathbf{x})$ is intractable to compute, then we may do search for the closest $q(\mathbf{z})$ among tractable distributions by solving optimization problem or represent this distribution by the neural network as in ~\cite{kingma2013auto,rezende2014stochastic}. In that case, $q(\mathbf{z})$ becomes variational approximation to $p_{\boldsymbol{\theta}}(\mathbf{z}|\mathbf{x})$.
\chapter{Variational Autoencoders}
\section{Variational Autoencoder\label{sec:vae_description}}
\subsection{Problem Setting}
We wish to be able to perform efficient approximate inference and learning in probabilistic directed latent variable models with continuous latent variables and/or parameters that have intractable posterior distributions and big datasets. In this case, the conventional EM algorithm will not work since the E step requires the posterior distribution to be known and tractable. If it is not, then marginal likelihood is also intractable. The mean-field Variational Bayes approach requires closed form solutions to the approximate posterior $p_{\boldsymbol{\theta}}(\mathbf{z}|\mathbf{x})$, which can be intractable in general case when dealing with large datasets and complicated likelihood $p_{\boldsymbol{\theta}}(\mathbf{x}|\mathbf{z})$ that can be represented by neural networks. Basically, we need to obtain this distribution to perform E step
\begin{equation}\label{eq:vae_elbo_deriv}
\begin{aligned}
q(\mathbf{z})=\prod_{i = 1}^{n}q(\mathbf{z}_{i})=\prod_{i = 1}^{n} p_{\boldsymbol{\theta}}(\mathbf{z}_{i}|\mathbf{x}_{i})=\prod_{i = 1}^{n}\frac{p_{\boldsymbol{\theta}}(\mathbf{z}_{i}|\mathbf{x}_{i})p_{\boldsymbol{\theta}}(\mathbf{z}_{i})}{\int p_{\boldsymbol{\theta}}(\mathbf{z}_{i}|\mathbf{x}_{i})p_{\boldsymbol{\theta}}(\mathbf{z}) dz}.
\end{aligned}
\end{equation}
However, in general setting the denominator is intractable. 

In ~\cite{kingma2013auto,rezende2014stochastic} the authors proposed solution to the problems of the aforementioned scenario. Lets consider such setting of probabilistic latent variable models when we have dataset $\mathbf{X}$= $\{\mathbf{x}^{(i)} \}_{i=1}^{N}$ of $\textit{N}$ i.i.d. data samples of some observed random variable $\mathbf{x}$. Under assumption that the observed data came from some generative  random process, involving latent random variable $\mathbf{z}$, the process is defined as follows: $\mathbf{z}$ comes from prior distribution $p_{\boldsymbol{\theta}^*}(\mathbf{z})$ and the observed sample $\mathbf{x}$ is generated by some process defined by $p_{\boldsymbol{\theta}^*}(\mathbf{x}| \mathbf{z})$. We assume that this two distributions that define generative process come from parametric families $p_{\boldsymbol{\theta}}(\mathbf{x}| \mathbf{z})$ and $p_{\boldsymbol{\theta}}(\mathbf{z})$ that are differentiable almost everywhere. In this setting, true parameters $\boldsymbol{\theta}^*$ and $\mathbf{z}$ are unknown. 
\subsection{Model Definition}
The ~\cite{kingma2013auto,rezende2014stochastic} authors propose a scalable solution for learning LVM that addresses intractability of posterior distributions and big datasets. They propose approximate ML or MAP estimation for the parameters $\boldsymbol{\theta}$ that allow one to analyze some hidden process as well as generate new data samples by mimicking the true generative process defined by $p_{\boldsymbol{\theta}^*}(\mathbf{x}| \mathbf{z})$ and $p_{\boldsymbol{\theta}^*}(\mathbf{z})$. Also, they provide method for efficient approximate posterior inference of $\mathbf{z}$ given $\mathbf{x}$.

The core idea behind the proposed approach is to represent intractable posterior distribution $p_{\boldsymbol{\theta}}(\mathbf{z}|\mathbf{x})$ by a flexible variational approximation $q_{\boldsymbol{\phi}}( \mathbf{z}|\mathbf{x})$ (a neural network). Thus, authors define $\textit{recognition model}$ as a probabilistic $\textit{encoder}$  $q_{\boldsymbol{\phi}}( \mathbf{z}|\mathbf{x})$ and $p_{\boldsymbol{\theta}}(\mathbf{x}|\mathbf{z})$ as probabilistic $\textit{decoder}$, both represented by neural networks with parameters $\boldsymbol{\phi}$ and  $\boldsymbol{\theta}$ respectively. The resulting directed probabilistic model is depicted in Figure~\ref{fig:vae_graph}.

\begin{figure}[h]
	\centering
	\begin{subfigure}[h]{0.45\textwidth}
		\includegraphics[width=\textwidth]{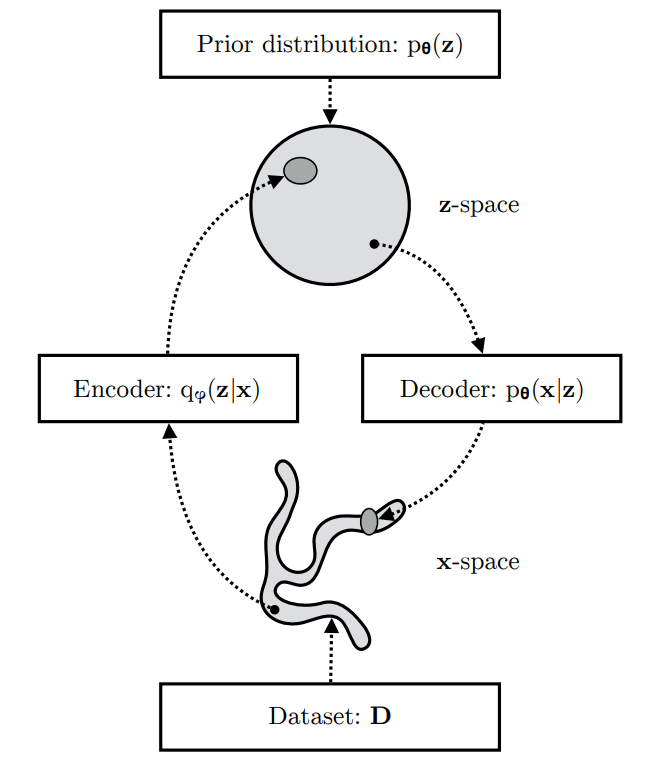}
		\caption{}
	\end{subfigure}
	\begin{subfigure}[h]{0.25\textwidth}
		\includegraphics[width=\textwidth]{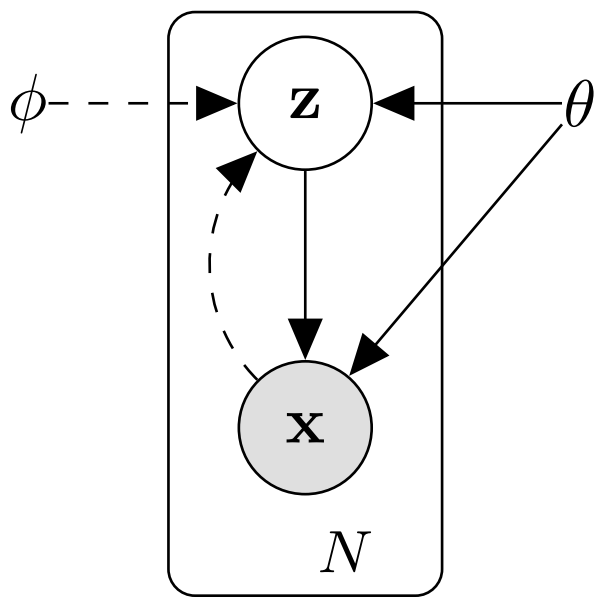}
		\caption{}
	\end{subfigure}
	\caption{Graphical directed latent variable model of variational autoencoder (from \cite{kingma2017variational} (a) and \cite{kingma2013auto} (b)). The generative model learns $p_{\boldsymbol{\theta}}(\mathbf{x},\mathbf{z})$, that factorizes into a decoder $p_{\boldsymbol{\theta}}(\mathbf{x}| \mathbf{z})$ and a prior $p_{\boldsymbol{\theta}}(\mathbf{z})$ distribution. The encoder $q_{\boldsymbol{\phi}}( \mathbf{z}|\mathbf{x})$ approximates intractable posterior $p_{\boldsymbol{\theta}}( \mathbf{z}|\mathbf{x})$ of the generative model.}
\label{fig:vae_graph}
\end{figure}

In this setting, we again wish to maximize marginal data log-likelihood $\textrm{log }p_{\boldsymbol{\theta}}(\mathbf{x})$ by maximizing its variational lower bound. We can derive variational lower bound (ELBO) on true data log-likelihood in the current setting as follows
\begin{equation}\label{eq:vae_elbo_deriv}
\begin{aligned}
    \textrm{log }p_{\boldsymbol{\theta}}(\mathbf{x})  &= \mathbb{E}_{q_{\boldsymbol{\phi}}( \mathbf{z}|\mathbf{x})} \Big[ \textrm{log }p_{\boldsymbol{\theta}}(\mathbf{x}) \Big] \\
     &= \mathbb{E}_{q_{\boldsymbol{\phi}}( \mathbf{z}|\mathbf{x})} \bigg[ \textrm{log} \left\{ \frac{p_{\boldsymbol{\theta}}(\mathbf{x}, \mathbf{z})}{q_{\boldsymbol{\phi}}( \mathbf{z}|\mathbf{x})} \frac{q_{\boldsymbol{\phi}}( \mathbf{z}|\mathbf{x})}{p_{\boldsymbol{\theta}}(\mathbf{z}| \mathbf{x})} \right\} \bigg] \\
     &= \mathbb{E}_{q_{\boldsymbol{\phi}}( \mathbf{z}|\mathbf{x})} \bigg[ \textrm{log} \left\{ \frac{p_{\boldsymbol{\theta}}(\mathbf{x}, \mathbf{z})}{q_{\boldsymbol{\phi}}( \mathbf{z}|\mathbf{x})}  \right\}\bigg] + 
    \mathbb{E}_{q_{\boldsymbol{\phi}}( \mathbf{z}|\mathbf{x})} \bigg[ \textrm{log} \left\{ \frac{q_{\boldsymbol{\phi}}( \mathbf{z}|\mathbf{x})}{p_{\boldsymbol{\theta}}(\mathbf{z}| \mathbf{x})} \right\} \bigg] \\
     &= \mathcal{L}(\boldsymbol{\theta},\boldsymbol{\phi}) + D_{\KL}\KLdel{q_{\boldsymbol{\phi}}( \mathbf{z}|\mathbf{x})}{p_{\boldsymbol{\theta}}(\mathbf{z}|\mathbf{x})}.
\end{aligned}
\end{equation}
Then, the variational lower bound can be decomposed as
\begin{equation}\label{eq:vae_elbo_decomp}
\mathcal{L}(\boldsymbol{\theta},\boldsymbol{\phi}) = \mathbb{E}_{q_{\boldsymbol{\phi}}( \mathbf{z}|\mathbf{x})} \bigg[ \textrm{log} \left\{ \frac{p_{\boldsymbol{\theta}}(\mathbf{x}| \mathbf{z})p_{\boldsymbol{\theta}}(\mathbf{x})}{q_{\boldsymbol{\phi}}( \mathbf{z}|\mathbf{x})}  \right\}\bigg] = \mathbb{E}_{q_{\boldsymbol{\phi}}( \mathbf{z}|\mathbf{x})} [\log p_{\boldsymbol{\theta}}(\mathbf{x}|\mathbf{z})] - D_{\KL}\KLdel{q_{\boldsymbol{\phi}}( \mathbf{z}|\mathbf{x})}{p_{\boldsymbol{\theta}}(\mathbf{z})},
\end{equation}
which we wish to maximize with respect to parameters $\boldsymbol{\theta}$ and $\boldsymbol{\phi}$ to maximize the true marginal data log-likelihood. Thus, the final objective has form of
\begin{equation}\label{eq:vae_objective}
\boldsymbol{\theta^*},\ \boldsymbol{\phi}^* =\mathop{\mathrm{arg\,max}}_{\boldsymbol{\theta},\boldsymbol{\phi}}\mathcal{L}(\boldsymbol{\theta},\boldsymbol{\phi})=\mathop{\mathrm{arg\,max}}_{\boldsymbol{\theta},\boldsymbol{\phi}} \ \mathbb{E}_{q_{\boldsymbol{\phi}}( \mathbf{z}|\mathbf{x})} [\log p_{\boldsymbol{\theta}}(\mathbf{x}|\mathbf{z})] - D_{\KL}\KLdel{q_{\boldsymbol{\phi}}( \mathbf{z}|\mathbf{x})}{p_{\boldsymbol{\theta}}(\mathbf{z})}
\end{equation}
Optimizing this objective can be seen as minimizing reconstruction error (first term) and making the approximate distribution $q_{\boldsymbol{\phi}}( \mathbf{z}|\mathbf{x})$ (represented by encoder network) closer to the selected prior $p_{\boldsymbol{\theta}}(\mathbf{z})$ (second term). Thus, the resulting model in Figure~\ref{fig:vae_model} is trained to encode and decode the data sample as well as keeping the distribution of encoded representations close to the selected prior distribution.
\begin{figure}[h]
\centering
\includegraphics[width=0.4\linewidth]{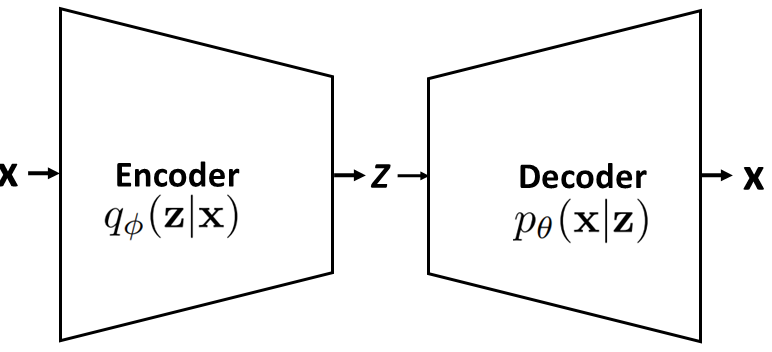}
\caption{Variational autoencoder model architecture}
\label{fig:vae_model}
\end{figure}

In the basic variational autoencoder setting, we let the prior distribution $p_{\boldsymbol{\theta}}(\mathbf{z})$ over latent variables be spherical multivariate isotropic multivariate Gaussian distribution $\mathcal{N}(\mathbf{z};0,I)$. On top of that, $p_{\boldsymbol{\theta}}(\mathbf{x}|\mathbf{z})$ can be multivariate Gaussian or Bernoulli distribution with parameters conditioned on $\mathbf{z}$ and computed using neural network. Regarding the approximation $q_{\boldsymbol{\phi}}( \mathbf{z}|\mathbf{x})$ to the true posterior, we assume that it has Gaussian form with diagonal covariance, meaning 
\begin{equation}\label{eq:approx_posterior_gaussian}
q_{\boldsymbol{\phi}}( \mathbf{z}|\mathbf{x})=\mathcal{N}(\mathbf{z}; \boldsymbol{\mu}(\mathbf{x}), \boldsymbol{\sigma}^2(\mathbf{x})I).
\end{equation}
For this distribution, $\boldsymbol{\mu}$ and $\boldsymbol{\sigma}$ are the ouptputs of the encoder neural network with parameters $\boldsymbol{\phi}$ that is conditioned on the input sample $\mathbf{x}$. 

Fitting the directed latent variable model represented by variational autoencoder results in generative and inference models. The inference model $q_{\boldsymbol{\phi}}( \mathbf{z}|\mathbf{x})$ can be used for the purposes of data representation tasks and for tasks related to semi-supervised training. The generative model $p_{\boldsymbol{\theta}}(\mathbf{x}|\mathbf{z})$ $p_{\boldsymbol{\theta}}(\mathbf{z})$ can be used for producing new data that resembles original data samples. It can be done by sampling $\mathbf{z}$ from prior distribution $p_{\boldsymbol{\theta}}(\mathbf{z})$ and passing them to the decoder network. The decoder and encoder networks together can be used for such tasks as image denoising, inpainting, and super-resolution.
\subsection{Connection to EM algorithm}
Training framework of variational autoencoders (VAE) can be interpreted as an extension of EM algorithm when the E step can not be performed due to intractability of posterior distribution $p_{\boldsymbol{\theta}}( \mathbf{z}|\mathbf{x})$. In the setting of VAE, E and M steps can be seen as maximizing ELBO with respect to parameters $\boldsymbol{\phi}$ and $\boldsymbol{\theta}$ respectively. The reasoning is next, we wish to maximize marginal log-likelihood of the data that is lower bounded by ELBO since from Equation~\ref{eq:vae_elbo_decomp} we know that
\begin{equation}\label{eq:vae_elbo_em}
\begin{aligned}
    \textrm{log }p_{\boldsymbol{\theta}}(\mathbf{x}) =  \mathcal{L}(\boldsymbol{\theta},\boldsymbol{\phi}) + D_{\KL}\KLdel{q_{\boldsymbol{\phi}}( \mathbf{z}|\mathbf{x})}{p_{\boldsymbol{\theta}}(\mathbf{z}|\mathbf{x})}.
\end{aligned}
\end{equation}
Having this equation, we can reformulate EM algorithm in terms of VAE framework:
\begin{itemize}
  \item \textit{E-step}. For each $\mathbf{x}$ $\in$ $\mathcal{D}$ we wish compute approximate posterior which is equivalent to making ELBO tight by:
  \begin{equation} \label{eq:em_vae_elbo_e_step}
      q_{\boldsymbol{\phi}}( \mathbf{z}|\mathbf{x})=\mathop{\mathrm{arg\,max}}_{\boldsymbol{\phi}}\mathcal{L}(\boldsymbol{\theta},\boldsymbol{\phi})=\mathop{\mathrm{arg\,min}}_{\boldsymbol{\phi}}D_{\KL}\KLdel{q_{\boldsymbol{\phi}}( \mathbf{z}|\mathbf{x})}{p_{\boldsymbol{\theta}}(\mathbf{z}|\mathbf{x})}
  \end{equation}
  Which is true, since the $\textrm{log }p_{\boldsymbol{\theta}}(\mathbf{x})$ does not depend on parameters $\boldsymbol{\phi}$ and thus maximizing ELBO with respect to this parameters will minimize the KL divergence term between true posterior and approximation.
  \item \textit{M-step}. Having ELBO tight to true data log-likelihood, we can find the parameters $\boldsymbol{\theta}$ that maximize it
  \begin{equation} \label{eq:em_vae_elbo_m_step}
      p_{\boldsymbol{\theta}}(\mathbf{x}|\mathbf{z})=\mathop{\mathrm{arg\,max}}_{\boldsymbol{\theta}}\mathcal{L}(\boldsymbol{\theta},\boldsymbol{\phi})=\mathop{\mathrm{arg\,max}}_{\boldsymbol{\theta}}\log p_{\boldsymbol{\theta}}(\mathbf{x}).
  \end{equation}
\end{itemize}
\subsection{Gradient Estimators and Reparametrization Trick}
As we mentioned before, we wish to optimize ELBO $\mathcal{L}(\boldsymbol{\theta},\boldsymbol{\phi})$ both with respect to parameters $\boldsymbol{\theta}$ and $\boldsymbol{\phi}$. The optimization w.r.t. $\boldsymbol{\theta}$ can be performed easily using Monte Carlo gradient estimator:
\begin{align}
    \nabla_{\boldsymbol{\theta}} \mathcal{L}(\boldsymbol{\theta},\boldsymbol{\phi}) &= \mathbb{E}_{q_{\boldsymbol{\phi}}( \mathbf{z}|\mathbf{x})} \Big[ \nabla_{\boldsymbol{\theta}} \left\{ \textrm{log }  p_{\boldsymbol{\theta}}(\mathbf{x}, \mathbf{z})- \textrm{log }q_{\boldsymbol{\phi}}( \mathbf{z}|\mathbf{x})  \right\}\Big] \\
    & \approx \frac{1}{N} \underset{n=1}{\overset{N}{\sum}} \nabla_{\boldsymbol{\theta}} \left\{ \textrm{log }  p_{\boldsymbol{\theta}}(\mathbf{x}, \mathbf{z}^{(n)}) \right\}, \ \textrm{for}\; \mathbf{z}^{(n)} \sim q_{\boldsymbol{\phi}}( \mathbf{z}|\mathbf{x}).
\end{align}

However, the usual Monte Carlo gradient estimator for parameters $\boldsymbol{\phi}$ will suffer from high variance. In stochastic computation graphs, which include both deterministic and stochastic nodes, the backpropagation algorithm has is not straight forward to perform. Therefore, in ~\cite{kingma2013auto} authors employ so called reparametrization trick (see Figure~\ref{fig:rep_trick}) to efficiently propagate gradients through stochastic nodes. In particular case of VAE, authors represent sampling procedure $\mathbf{z} \sim q_{\boldsymbol{\phi}}( \mathbf{z}|\mathbf{x})$ as a differentiable transformation of independent random variable $\boldsymbol{\epsilon} \sim p(\boldsymbol{\epsilon})$ such that $\mathbf{z}=\mathbf{g}(\boldsymbol{\epsilon}, \boldsymbol{\phi}, \mathbf{x})$. Then Monte Carlo gradient estimator can be efficiently used again:
\begin{align}
    \nabla_{\boldsymbol{\theta}} \mathcal{L}(\boldsymbol{\theta},\boldsymbol{\phi}) &= \nabla_{\boldsymbol{\phi}} \mathbb{E}_{p(\boldsymbol{\epsilon})} \Big[   \textrm{log }  p_{\boldsymbol{\theta}}(\mathbf{x}, \mathbf{z})- \textrm{log }q_{\boldsymbol{\phi}}( \mathbf{z}|\mathbf{x})  \Big] \\
    &= \mathbb{E}_{p(\boldsymbol{\epsilon})} \Big[ \nabla_{\boldsymbol{\phi}} \left\{  \textrm{log }  p_{\boldsymbol{\theta}}(\mathbf{x}, \mathbf{z})- \textrm{log }q_{\boldsymbol{\phi}}( \mathbf{z}|\mathbf{x}) \right\}  \Big] \\
    & \approx \frac{1}{N} \underset{n=1}{\overset{N}{\sum}} \nabla_{\boldsymbol{\phi}} \left\{ \textrm{log }  p_{\boldsymbol{\theta}}(\mathbf{x}, \mathbf{z}^{(n)}) -\textrm{log }q_{\boldsymbol{\phi}}( \mathbf{z}^{(n)}|\mathbf{x}) \right\}, \\
    & \quad \textrm{for}\; \mathbf{z}^{(n)}= \mathbf{g}(\boldsymbol{\epsilon}^{(n)}, \boldsymbol{\phi}, \mathbf{x})
    \ \textrm{and}\; \boldsymbol{\epsilon}^{(n)} \sim p(\boldsymbol{\epsilon}).
\end{align}
For the particular case of variational autoencoder with Gaussian multivariate prior distribution where
\begin{equation}\label{eq:approx_posterior_gaussian_}
q_{\boldsymbol{\phi}}( \mathbf{z}|\mathbf{x})=\mathcal{N}(\mathbf{z}; \boldsymbol{\mu}(\mathbf{x}), \boldsymbol{\sigma}^2(\mathbf{x})I).
\end{equation}
the reparametrization trick is done as
\begin{align}
\textrm{for}\; \mathbf{z}^{(n)}= \boldsymbol{\mu}^{(n)} + \boldsymbol{\sigma}^{(n)}\odot\boldsymbol{\epsilon}^{(n)}
    \ \textrm{and}\; \boldsymbol{\epsilon}^{(n)} \sim \mathcal{N}(0,I).
\end{align}

\begin{figure}[ht]
\centering
\includegraphics[width=0.6\linewidth]{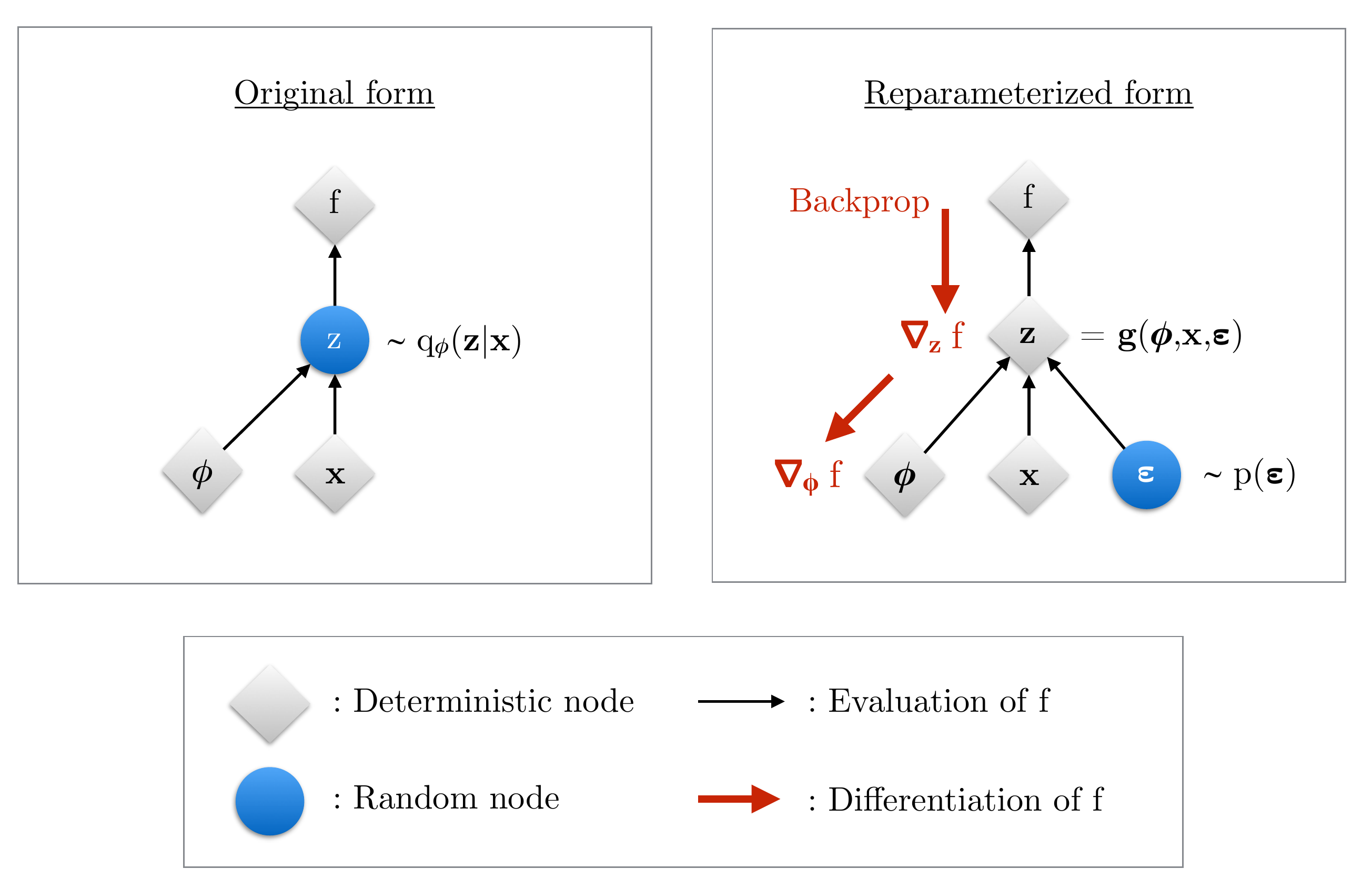}
\caption{Reparametrization trick (from \cite{kingma2017variational}).}
\label{fig:rep_trick}
\end{figure}

\subsection{Representation Learning and Visualizations}

It is known that having only the reconstruction criterion in autoencoders is not sufficient for learning useful representations \cite{bengio2013representation}. Initially, different regularization approaches have been introduced to make autoencoder models learn useful representations: denoising, contractive, and sparse autoencoder variants. The VAE objective contains a regularization term that comes from the variational lower bound itself which does not specify that it is required to learn useful representations in its original form.

The authors in \cite{kingma2013auto} select a low-dimensional latent space for VAE
and use the learned encoders (recognition model) to project high-dimensional data to a low-dimensional manifold. See Figure~\ref{fig:2dmanifolds} for visualisations of the 2D latent manifolds for the MNIST and Frey Face datasets.

\begin{figure}[h]
	\centering
	\begin{subfigure}[h]{0.30\textwidth}
		\includegraphics[width=\textwidth]{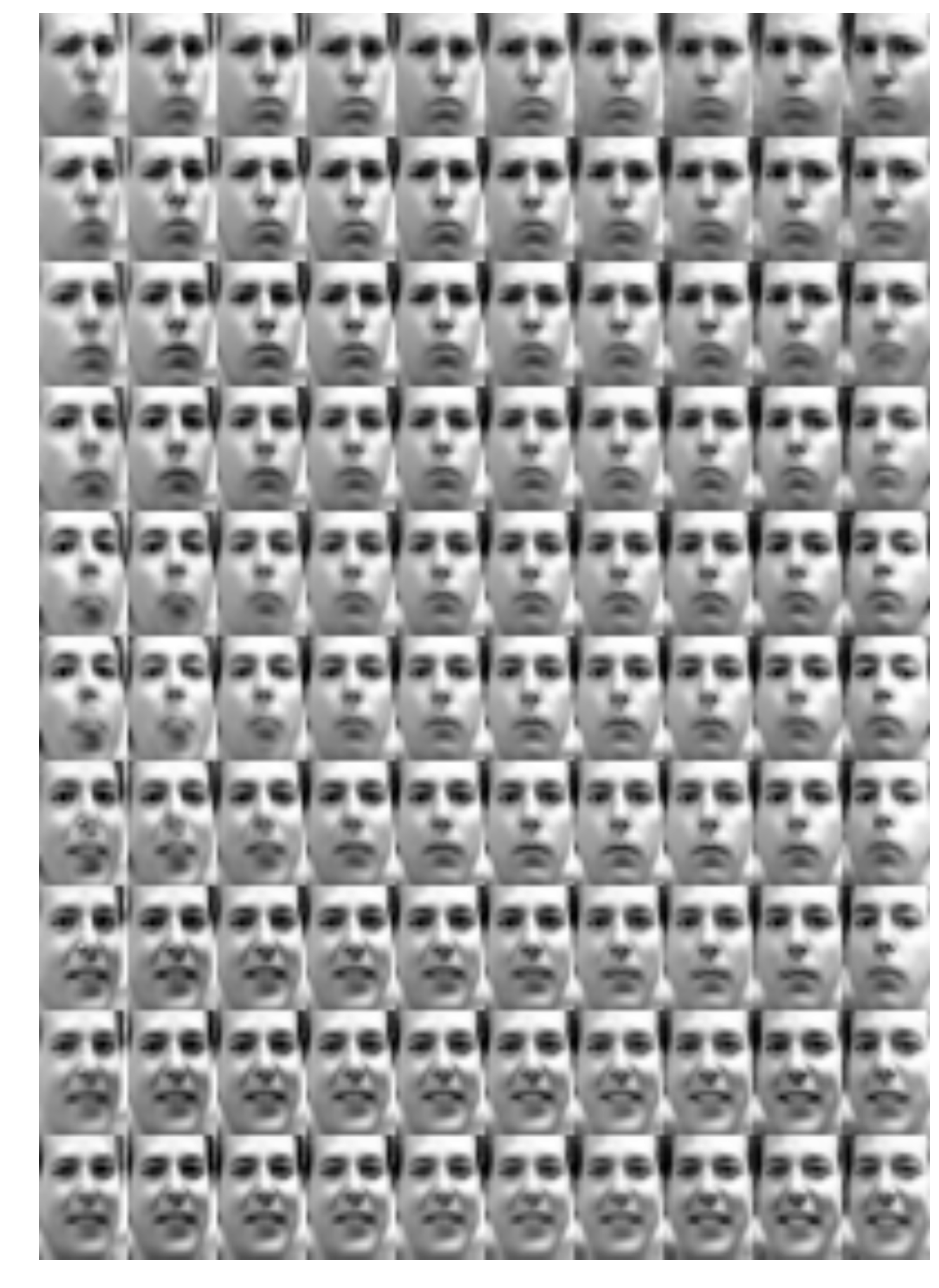}
		\caption{Learned Frey Face manifold}
	\end{subfigure}
	\begin{subfigure}[h]{0.45\textwidth}
		\includegraphics[width=\textwidth]{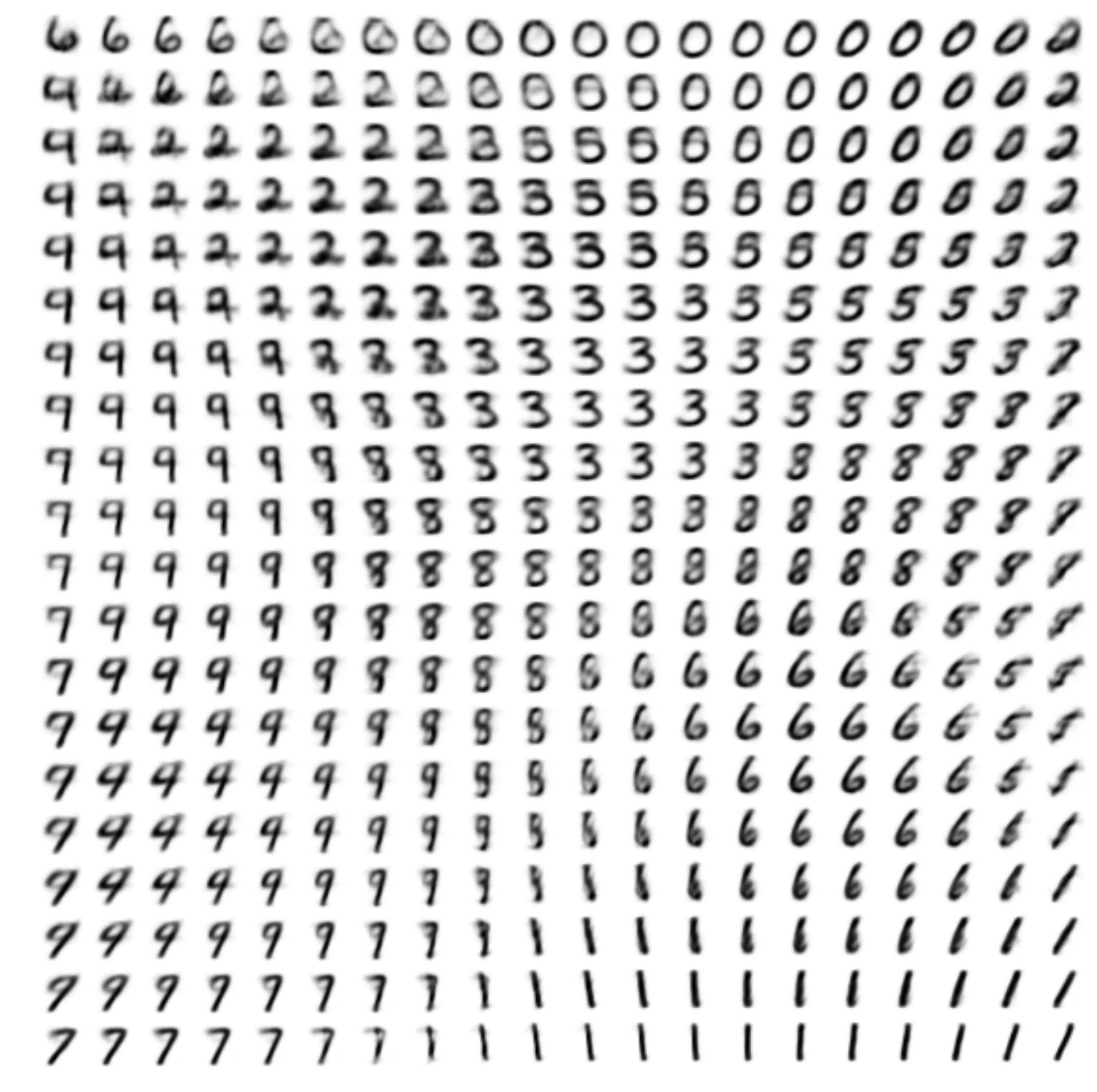}
		\caption{Learned MNIST manifold}
	\end{subfigure}
	\caption{Visualisations of learned data manifold for VAE generative mode
ls with two-dimensional latent space (from \cite{kingma2013auto}).}
\label{fig:2dmanifolds}
\end{figure}
\section{Beta Variational Autoencoder ($\beta$-VAE)\label{sec:beta_vae}}
\subsection{Training Framework}
In \cite{higgins2017beta,burgess2018understanding} the authors proposed modification for original variational autoencoder framework called $\beta$-VAE for automated discovery of interpretable factorised latent representations from raw image data in a completely unsupervised manner that improves representation learing in VAE. On top of that, in \cite{burgess2018understanding} authors introduced theoretical insights for the sources of disentangled and interpretable representations in VAE. 

The core idea is to introduce hyperparameter $\beta$ that modulates the learning constraints in the trained model. This constraints control the capacity of the latent information channel and learning of statistically independent latent factors. $\beta$-VAE  with $\beta$=1 corresponds to the original VAE framework. With $\beta$ $>$ 1 the model is forced to learn a more efficient and disentangled latent representation of the data. The resulting modification to the original ELBO objective is next
\begin{align} \label{eq_beta_vae}\mathbf{}
    \mathcal{L}(\theta, \phi; \beta)=\mathbb{E}_{q_{\boldsymbol{\phi}}( \mathbf{z}|\mathbf{x})} [\log p_{\boldsymbol{\theta}}(\mathbf{x}|\mathbf{z})] - \beta D_{\KL}\KLdel{q_{\boldsymbol{\phi}}( \mathbf{z}|\mathbf{x})}{p_{\boldsymbol{\theta}}(\mathbf{z})}
\end{align}
where $\beta$ is KKT multiplier that acts as regularization coefficient that constrains the capacity of the latent information channel $\mathbf{z}$ and puts implicit independence pressure on the approximate posterior because of the isotropic nature of the selected Gaussian prior distribution.
\subsection{Information Theoretic Perspective}
In \cite{burgess2018understanding} authors consider approximate posterior distribution $q_{\boldsymbol{\phi}}( \mathbf{z}|\mathbf{x})$ as an as an information bottleneck for the reconstruction task that is learned by
\begin{align} \label{eq_beta_vae_reconstr}
    \max_{\boldsymbol{\theta}}\mathbb{E}_{q_{\boldsymbol{\phi}}( \mathbf{z}|\mathbf{x})} [\log p_{\boldsymbol{\theta}}(\mathbf{x}|\mathbf{z})].
\end{align}
Authors argue that in $\beta$-VAE objective, $q_{\boldsymbol{\phi}}( \mathbf{z}|\mathbf{x})$ is trained to transmit information efficiently about the observations $\mathbf{x}$ by joint minimization of $\beta$-weighted KL-divergence and marginal data log-likelihood maximization.

In this setting, the posterior $q_{\boldsymbol{\phi}}( \mathbf{z}|\mathbf{x})$ matched to Gaussian prior $p_{\boldsymbol{\theta}}(z_i) = \mathcal{N}(0, 1)$. For each latent unit $z_i$, we can take an information theoretic perspective with mean-field approach and think of $q_{\boldsymbol{\phi}}( \mathbf{z}|\mathbf{x})$ as a set of independent Gaussian channels $z_i$, each noisily transmitting information about the encoded data samples $\mathbf{x}$. From this perspective, the term $D_{\KL}\KLdel{q_{\boldsymbol{\phi}}( \mathbf{z}|\mathbf{x})}{p_{\boldsymbol{\theta}}(\mathbf{z})}$ of the objective function ~\ref{eq_beta_vae} can be interpreted as an upper bound on the information that can be transmitted through the latent channels per data sample. The KL divergence is zero when $q(z_i|\mathbf{x})=p_{\boldsymbol{\theta}}(\mathbf{z})$, i.e $\mu_i$ is always zero, and $\sigma_i$ always 1, meaning the latent channels $z_i$ have zero capacity. 
\subsection{Representation Learning}
$\beta$-VAE aligns latent dimensions with components that make different contributions to reconstruction. By reasoning in \cite{burgess2018understanding}, reconstruction task under this bottleneck aligns the data observed points on a set of representational axes where nearby points on the axes are also close in data space. This is due to a strong pressure for overlapping posteriors that forces $\beta$-VAE to find a representation space that preserves the locality of points on the data manifold. The VAE model trained in this framework finds disentangled representations that align with generative factors of data since it is able to find latent components which make different contributions to the log-likelihood term in the objective ~\ref{eq_beta_vae}. These latent components correspond to properties in observed data that are somehow qualitatively different, and thus may embed into the generative factors in the data. In Figure \ref{fig:beta_vae_vae_comparison} you can see comparison of learned feature space disentanglment of VAE and $\beta$-VAE models.
\begin{figure}[h]
\centering
\includegraphics[width=0.5\linewidth]{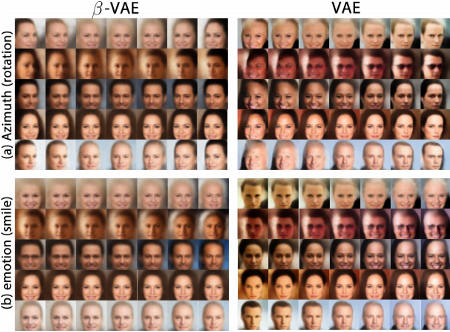}
\caption{Manipulating latent variables on celebA of VAE and $\beta$-VAE models (from \cite{higgins2017beta}).}
\label{fig:beta_vae_vae_comparison}
\end{figure}
\subsection{Capacity Control Increase}
On top of constraining KL divergence term impact in $\beta$-VAE by weighting coefficient $\beta$, in \cite{burgess2018understanding} authors propose to constrain this term by target coefficient $C$ that is increased during VAE training with objective
 \begin{align} \label{cc_bvae_objective}
    \mathcal{L}(\theta, \phi; C)  =  \mathbb{E}_{q_{\boldsymbol{\phi}}( \mathbf{z}|\mathbf{x})} [\log p_{\boldsymbol{\theta}}(\mathbf{x}|\mathbf{z})] - \gamma \ |D_{\KL}\KLdel{q_{\boldsymbol{\phi}}( \mathbf{z}|\mathbf{x})}{p_{\boldsymbol{\theta}}(\mathbf{z})} - C|.
\end{align}
The intuition is gradually adding more latent encoding capacity, enabling more learned factors of variation to be represented while retaining disentangling in previously
learned factors.

\chapter{Variational Mutual Information Maximization for VAE}
\section{Motivation\label{sec:motivation}}
\subsection{Latent Representations}
Latent variable models such as variational autoencoders \cite{kingma2013auto,rezende2014stochastic} is a powerful approach to generative modeling of complicated distributions. This models are defined as  $p_{\boldsymbol{\theta}}(\mathbf{x}|\mathbf{z})p_{\boldsymbol{\theta}}(\mathbf{z})$ and trained in the framework of maximum likelihood estimation with marginal data log-likelihood. Although the latent variables $\mathbf{z}$ are not observed,
they could provide a high-level representation and align with underlying generative factors of the observations $\mathbf{x}$. Thus, this latent variables can serve as useful representations for various tasks in machine learning\cite{bengio2013representation}. 

Since variational autoencoder model can be fitted without labeled data, it can be used for unsupervised and semi-supervised learning tasks, which could be an important part of machine learning system. Using ELBO objective alone could not be enough to force the model to learn useful representations by $\mathbf{z}$. The amount of useful and interpretable information in $\mathbf{z}$ directly relies on the expressiveness of the selected parametric family of models $p_{\boldsymbol{\theta}}(\mathbf{x},\mathbf{z})$ with respect to the true data distribution \cite{chen2016variational,chen2018isolating}, assumtions related to the form of underlying genearative process, and model initialization. This leads to difficulties of VAE use in applications such as natural language processing and application involving discrete data.
\subsection{ELBO Limitations}
VAE framework is capable of learning interpretable, disentangled and useful representations by latent codes $\mathbf{z}$ from unlabeled or weakly labeled data. Despite impressive results in various tasks such as in \cite{higgins2017beta}, there are ones in which the use of VAE use is still challenging. For instance, in \cite{bowman2016generating} the authors found that it is possible for decoder to fail  learning useful representations when approximate posterior distribution collapsed completely to the selected prior distribution, meaning $q_{\boldsymbol{\phi}}( \mathbf{z}|\mathbf{x})=p_{\boldsymbol{\theta}}(\mathbf{z})$.

The reason for such issues is that original ELBO objective does not provide explicit measure for the quality of representations that latent variable learn since the latent variable is marginalized in the final ELBO ojective. It is so since the marginalized log-likelihood $\log p_{\boldsymbol{\theta}}(\mathbf{x})$ that is lower bounded by ELBO is solely function of $\mathbf{x}$. At the same time, the full model is represented by $p_{\boldsymbol{\theta}}(\mathbf{x}|\mathbf{z})p_{\boldsymbol{\theta}}(\mathbf{z})=p_{\boldsymbol{\theta}}(\mathbf{x},\mathbf{z})$. In the setting of VAE, the approximation to the marginal log-likelihood is maximized in hope that it will recover true generative process with interpretable and useful latent representations. This approach may lead to feasible results when there are strong constraints on joint distribution. On the other hand, training high-capacity decoders can result in ignoring conditioning on $\mathbf{z}$ ($p_{\boldsymbol{\theta}}(\mathbf{x}|\mathbf{z})=p_{\boldsymbol{\theta}}(\mathbf{x})$) but still have high ELBO and marginal likelihood $\log p_{\boldsymbol{\theta}}(\mathbf{x})$ \cite{bowman2016generating,chen2016variational}. \textbf{Thus, obtaining high ELBO does not necessarily lead to good quality latent representations}. On top of that, in \cite{alemi2018fixing} the authors showed that VAE models from the same parametric family and with identical ELBO can have different quantitative and qualitative characteristics. 
\section{Definition of Proposed Framework\label{sec:framework}}

The key idea of our approach is to improve learned representations in VAE and overcome ELBO limitations by providing an explicit control technique for relations between observations and latent codes. The aim of the proposed approach is to maximize mutual information (MI) between latent variables $\mathbf{z}$ and observations $\mathbf{x}$.  Unfortunately, exact MI computing is hard and may be intractable. To overcome this, our framework employs Variational Information Maximization \cite{barber2003algorithm} to obtain lower bound on true MI. The obtained lower bound on MI is used as a regularizer in the addition to the original VAE objective (ELBO) to force the latent codes to have strong relationship with observations, prevent the model from ignoring these codes, and learn useful representations.
\subsection{Mutual Information}
In information theory, mutual information (MI) between random variables $\mathbf{x}$ and $\mathbf{z}$ , $I(\mathbf{z}; \mathbf{x})$, measures the amount of information that can be inferred using knowledge of one random variable about another one. Mutual information can be formulated as the difference of two entropy terms:
\begin{equation}
\label{information-z-x}
    I(\mathbf{z}; \mathbf{x}) = H(\mathbf{z}) - H(\mathbf{z} | \mathbf{x}).
\end{equation}
The entropy of random variable $H(\mathbf{z})$ can be seen as a measure of uncertainty about this variable. For instance, a discrete uniform random variable will have higher entropy than the same one with probabilities distributed in a different manner. Thus, the MI can be seen as the amount of uncertainty about one random variable that is left when the value of other one is revealed. The other formulation of MI is represented as 
\begin{equation}
\label{information-z-x}
    I(\mathbf{z}; \mathbf{x}) = \mathbb{E}_{p( \mathbf{z},\mathbf{x})} \bigg[ \textrm{log} \left\{ \frac{p(\mathbf{z}|\mathbf{x})}{p(\mathbf{z})} \right\} \bigg]
\end{equation}
\subsection{Variational Mutual Information Lower Bound}
Following the reasoning in \cite{barber2003algorithm,chen2016infogan} we can derive variational lower bound on true mutual information as follows
\begin{equation}
\label{information_lower_bound_derivation}
\begin{aligned}
    I(\mathbf{z}; \mathbf{x}) = \mathbb{E}_{p( \mathbf{z},\mathbf{x})} \bigg[ \textrm{log} \{ \frac{p(\mathbf{z}|\mathbf{x})}{p(\mathbf{z})} \} \bigg]
    &= \mathbb{E}_{p(\mathbf{z},\mathbf{x})} \bigg[ \textrm{log} \left\{ \frac{p(\mathbf{z}|\mathbf{x})}{Q(\mathbf{z}|\mathbf{x})} \frac{Q(\mathbf{z}|\mathbf{x})}{p(\mathbf{z})} \right\} \bigg] \\
     &= \mathbb{E}_{p(\mathbf{x})} \bigg[ \mathbb{E}_{p( \mathbf{z}|\mathbf{x})} \bigg[ \textrm{log} \left\{ \frac{p(\mathbf{z}|\mathbf{x})}{Q(\mathbf{z}|\mathbf{x})} \right\}\bigg] + \mathbb{E}_{p( \mathbf{z}|\mathbf{x})} \bigg[ \textrm{log} \left\{ \frac{Q(\mathbf{z}|\mathbf{x})}{p(\mathbf{z})} \right\} \bigg] \\
     &= \mathbb{E}_{p(\mathbf{x})}[D_{\KL}\KLdel{p( \mathbf{z}|\mathbf{x})}{Q(\mathbf{z}|\mathbf{x})}] + \mathbb{E}_{p(\mathbf{z},\mathbf{x})}[\log Q(\mathbf{z}|\mathbf{x})] + H(\mathbf{z})\\
     &\geq \mathbb{E}_{p(\mathbf{z},\mathbf{x})}[\log Q(\mathbf{z}|\mathbf{x})] + H(\mathbf{z}),
\end{aligned}
\end{equation}
where Q is some auxiliary distribution and last inequality arises due to non-negativity of KL-divergence. The bound is tight when $Q(\mathbf{z}|\mathbf{x})=p(\mathbf{z}|\mathbf{x})$ which can be achieved by optimization with respect to the auxiliary distribution $Q$. Then we arrive to variational lower bound on MI defined as
\begin{equation}
\label{information_lower_bound_general}
\begin{aligned}
    I(\mathbf{z}; \mathbf{x}) \geq \mathbb{E}_{p(\mathbf{z},\mathbf{x})}[\log Q(\mathbf{z}|\mathbf{x})] + H(\mathbf{z}).
\end{aligned}
\end{equation}
\subsection{Variational Mutual Information Lower Bound for VAE}
In the setting of VAE we are aiming to fit latent variable model of the form $p_{\boldsymbol{\theta}}(\mathbf{x}|\mathbf{z})p_{\boldsymbol{\theta}}(\mathbf{z})$. Unsing insights from the previous chapter, we can define variational lower bound on mutual information between observed variables $\mathbf{x}$ and latent variables $\mathbf{z}$ in the setting of variational autoencoder as
\begin{equation}
\label{information_lower_bound_general_vae}
\begin{aligned}
    I(\mathbf{z}; \mathbf{x}) \geq \mathbb{E}_{p_{\boldsymbol{\theta}}(\mathbf{z},\mathbf{x})}[\log Q(\mathbf{z}|\mathbf{x})] + H(\mathbf{z})=\mathbb{E}_{p_{\boldsymbol{\theta}}(\mathbf{z})}[\mathbb{E}_{p_{\boldsymbol{\theta}}(\mathbf{x}|\mathbf{z})}[\log Q(\mathbf{z}|\mathbf{x})]] + H(\mathbf{z}).
\end{aligned}
\end{equation}
The problem with the obtained lowed bound is that it involves two nested expectations over prior and decoder distribution. To deal with it we can use next lemma
\begin{lemma}
\label{thelemma}
For random variables $X, Y$ and function $f(x, y)$ under suitable regularity conditions: 
\begin{equation}
    \mathbb{E}_{x \sim X, y \sim  Y|x} [f(x, y)] =
    \mathbb{E}_{x \sim X, y \sim  Y|x, x' \sim  X|y} [f(x', y)].
\end{equation}
\end{lemma}
\begin{proof}
(This proof is taken from ~\cite{Ford2018lemmaproof}).

Make expectations explicit: 
\begin{equation*}
        \mathbb{E}_{x \sim X, y \sim Y|x}[f(x, y)] = \mathbb{E}_{x \sim P(X)}\big[\mathbb{E}_{y \sim P(Y|X=x)}[f(x, y)]\big]
\end{equation*}

By definition of $P(Y|X=x)$ and $P(X|Y=y)$:
\begin{equation*}
         \mathbb{E}_{x \sim P(X)}\big[\mathbb{E}_{y \sim P(Y|X=x)}[f(x, y)]\big] = 
         \mathbb{E}_{x,y \sim P(X,Y)}[f(x, y)] = \mathbb{E}_{y \sim P(Y)}\big[\mathbb{E}_{x \sim P(X|Y=y)}[f(x, y)]\big]
\end{equation*}

Rename $x$ to $x'$:
\begin{equation*}
         \mathbb{E}_{y \sim P(Y)}\big[ \mathbb{E}_{x \sim P(X|Y=y)}[f(x, y)]\big] = 
         \mathbb{E}_{y \sim P(Y)}\big[ \mathbb{E}_{x' \sim P(X|Y=y)}[f(x', y)]\big]
\end{equation*}

By the law of total expectation:
\begin{equation*}
         \mathbb{E}_{y \sim P(Y)}\big[ \mathbb{E}_{x' \sim P(X|Y=y)}[f(x', y)]\big] = 
         \mathbb{E}_{x \sim P(X)}\Big[ \mathbb{E}_{y \sim P(Y|X=x)}\big[\mathbb{E}_{x' \sim P(X|Y=y)}[f(x', y)]\big]\Big]
\end{equation*}

Make expectations implicit:
\begin{equation*}
        \mathbb{E}_{x \sim P(X)}\Big[\mathbb{E}_{y \sim P(Y|X=x)}\big[\mathbb{E}_{x' \sim P(X|Y=y)}[f(x', y)]\big]\Big] = 
        \mathbb{E}_{x \sim X,y \sim Y|x,x' \sim X|y}[f(x', y)] 
\end{equation*}
\end{proof}
Using this lemma we can rewrite variational lower bound on MI for VAE as
\begin{equation}
\label{information_lower_bound_general_vae_improved}
\begin{aligned}
    I(\mathbf{z}; \mathbf{x}) \geq \mathbb{E}_{p_{\boldsymbol{\theta}}(\mathbf{z})}[\mathbb{E}_{p_{\boldsymbol{\theta}}(\mathbf{x}|\mathbf{z})}[\log Q(\mathbf{z}|\mathbf{x})]] + H(\mathbf{z})=\mathbb{E}_{\mathbf{z} \sim p_{\boldsymbol{\theta}}(\mathbf{z}), \mathbf{x} \sim p_{\boldsymbol{\theta}}(\mathbf{x}|\mathbf{z})}[\log Q(\mathbf{z}|\mathbf{x})]+H(\mathbf{z}).
\end{aligned}
\end{equation}
To make this lower bounding technique applicable to full VAE model, we substitute $p_{\boldsymbol{\theta}}(\mathbf{z})$ by approximate posterior distribution represented by encoder network $q_{\boldsymbol{\phi}}( \mathbf{z}|\mathbf{x})$. Finally, we arrive to variational MI lower bound estimate between latent variables and observations for a fixed VAE defined as
\begin{equation} \label{eq:mut_information_bound}
    \max_{Q} \ \mathbb{E}_{\mathbf{z} \sim q_{\boldsymbol{\phi}}(\mathbf{z}|\mathbf{x}), \mathbf{x} \sim  p_{\boldsymbol{\theta}}(\mathbf{x}|\mathbf{z})} [\log Q(\mathbf{z}|\mathbf{x})] + H(\mathbf{z}) \leq I(\mathbf{z};\mathbf{x}).
\end{equation}
which can be easily estimated using Monte Carlo sampling. Using this lower bound for a fixed VAE it is possible to evaluate mutual information between observations and latent variables. We represent the auxiliary distribution Q using neural network that takes output of the encoder as input.
\subsection{Variational Mutual Information Maximization Framework for VAE}
We wish to use the obtained lower bound on MI (Equation \ref{eq:mut_information_bound}) as the regularizer together with original VAE objective (ELBO) to force latent variables to have a strong relationship with observations, learn useful representations and prevent the VAE model from ignoring them. We define mutual information maximization regularizer $\mathit{MI}$ for variational autoencoder as
\begin{equation} \label{eq:mut_information_reg}
    \mathit{MI}(\theta, \phi, Q) = \mathbb{E}_{\mathbf{z} \sim q_{\boldsymbol{\phi}}(\mathbf{z}|\mathbf{x}), \mathbf{x} \sim  p_{\boldsymbol{\theta}}(\mathbf{x}|\mathbf{z})} [\log Q(\mathbf{z}|\mathbf{x})] + H(\mathbf{z})
\end{equation}
where Q is an auxiliary distribution represented by a neural network which takes the decoder output as input. We combine this regularizer with ELBO objective of VAE
\begin{equation}\label{eq:vae_objective_again}
\mathcal{L}(\boldsymbol{\theta},\boldsymbol{\phi})= \mathbb{E}_{q_{\boldsymbol{\phi}}( \mathbf{z}|\mathbf{x})} [\log p_{\boldsymbol{\theta}}(\mathbf{x}|\mathbf{z})] - D_{\KL}\KLdel{q_{\boldsymbol{\phi}}( \mathbf{z}|\mathbf{x})}{p_{\boldsymbol{\theta}}(\mathbf{z})}.
\end{equation}
Then, the final objective that we propose to use for training VAE with MI maximization is
\begin{equation}\label{eq:vmi_vae_objective}
\begin{split}
\mathcal{L}(\boldsymbol{\theta},\boldsymbol{\phi}, Q) &=\mathcal{L}(\boldsymbol{\theta},\boldsymbol{\phi}) + \lambda\mathit{MI}(\theta, \phi, Q) \\ &=\mathbb{E}_{q_{\boldsymbol{\phi}}( \mathbf{z}|\mathbf{x})} [\log p_{\boldsymbol{\theta}}(\mathbf{x}|\mathbf{z})] - D_{\KL}\KLdel{q_{\boldsymbol{\phi}}( \mathbf{z}|\mathbf{x})}{p_{\boldsymbol{\theta}}(\mathbf{z})} + \lambda\mathit{MI}(\theta, \phi, Q)
\end{split}
\end{equation}
where $\lambda$ is a scaling coefficient that controls the impact of $MI$ regularizer. Please see Figure~\ref{fig:main_model} for visualization of the proposed framework. We wish to maximize this objective with respect to the parameters of VAE $\boldsymbol{\theta}$ and $\boldsymbol{\phi}$ as well as to parameters of auxiliary network $Q$, meaning
\begin{equation}\label{eq:vmi_vae_objective_max}
\max_{\boldsymbol{\theta},\boldsymbol{\phi},Q} \ \mathcal{L}(\boldsymbol{\theta},\boldsymbol{\phi}) + \lambda\mathit{MI}(\theta, \phi, Q).
\end{equation}

For each training batch, first, we maximize the objective with respect to the auxiliary distribution $Q$ to make estimate of mutual information lower bound tighter. Then, we maximize it with respect to parameters of the VAE ($\boldsymbol{\theta},\boldsymbol{\phi}$) to train VAE using the $MI$ regularizer to maximize and control MI between latent codes and observations. Below, we provide the training procedure algorithm.
\begin{algorithm}[h]
\SetAlgoLined
$\theta, \phi, Q \gets$ Initialize parameters\;
\Repeat{convergence of parameters $\theta, \phi,Q$}{
      $\boldsymbol{X}^M \gets $ Random minibatch of $M$ datapoints (drawn from full dataset)\;
      maximize $\mathcal{L}(\boldsymbol{\theta},\boldsymbol{\phi}, Q)$ with respect to the auxiliary distribution $Q$\;
      maximize $\mathcal{L}(\boldsymbol{\theta},\boldsymbol{\phi}, Q)$ with respect to $\boldsymbol{\theta},\boldsymbol{\phi}$\;
    }
\Return $\boldsymbol{\theta},\boldsymbol{\phi}$
\caption{Training VAE with variational mutual information maximization}
\end{algorithm}
\begin{figure}[t]
\centering
\includegraphics[width=0.6\linewidth]{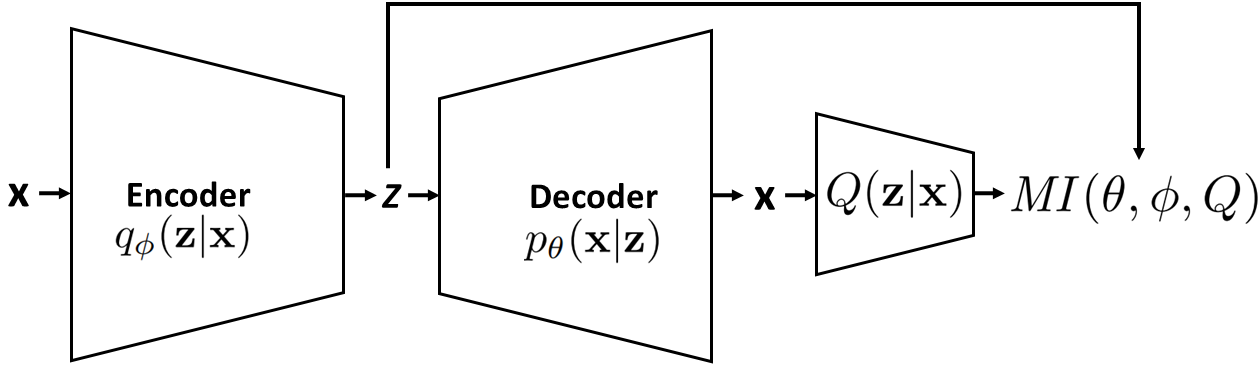}
\caption{Graphical description of the proposed model}
\label{fig:main_model}
\end{figure}
\section{Experimental Setting\label{sec:exp_setting}}
\subsection{VAE with Gaussian Latent Variable}
In this setting, we employ variational autoencoder model with 32-dimensional Gaussian spherical latent variable with prior $p_{\boldsymbol{\theta}}(\mathbf{z})=\mathcal{N}(0,1)$ and approximate posterior distribution as $q_{\boldsymbol{\phi}}( \mathbf{z}|\mathbf{x})=\mathcal{N}(\mathbf{z}; \boldsymbol{\mu}(\mathbf{x}), \boldsymbol{\sigma}^2(\mathbf{x})I)$. We train and compare two identically initialized networks with same hyperparameters on MNIST dataset. One is trained using only ELBO objective (Equation~ \ref{eq:vae_objective_again}) and the other with mutual information maximization (Equation~\ref{eq:vmi_vae_objective}). For the latter case, we selected only two components latent code vector forming sub-vector ($\mathbf{z}_{1}$,$\mathbf{z}_{2}$)=$\mathbf{\hat{z}}$. We selected them by visually inspecting influence from latent vector components in VAE trained using only ELBO to find  ones that had lesser impact on produced samples. Then, for MI maximization in this setting we define the $MI$ regularizer as
\begin{equation} \label{eq:mut_information_contin}
    \mathit{MI}(\boldsymbol{\theta}, \boldsymbol{\phi}, Q)=
    \mathbb{E}_{\mathbf{\hat{z}} \sim q_{\boldsymbol{\phi}}(\mathbf{\hat{z}}|\mathbf{x}), \mathbf{x} \sim  p_{\boldsymbol{\theta}}(\mathbf{x}|\mathbf{{z}})} [\log Q(\mathbf{\hat{z}}|\mathbf{x})] + H(\mathbf{\hat{z}}).
\end{equation}

We select only two components of the latent code since it is straightforward to illustrate their impact on observations in 2D visualizations by just manipulating 
their individual values without any latent space interpolations. Then, the resulting objective is
\begin{equation}\label{eq:vmi_vae_objective_gaussian}
\begin{split}
\mathcal{L}(\boldsymbol{\theta},\boldsymbol{\phi}, Q) &=\mathcal{L}(\boldsymbol{\theta},\boldsymbol{\phi}) + \lambda\mathit{MI}(\theta, \phi, Q)\\
&=\mathcal{L}(\boldsymbol{\theta},\boldsymbol{\phi}) + \lambda (\ \mathbb{E}_{\mathbf{\hat{z}} \sim q_{\boldsymbol{\phi}}(\mathbf{\hat{z}}|\mathbf{x}), \mathbf{x} \sim  p_{\boldsymbol{\theta}}(\mathbf{x}|\mathbf{{z}})} [\log Q(\mathbf{\hat{z}}|\mathbf{x})] + H(\mathbf{\hat{z}}) \ )
\end{split}
\end{equation}
\subsection{VAE with joint Gaussian and Discrete Latent Variable}
We also performed experiments on variational autoencoder models that involve discrete latent variable that potentially could learn and model categorical (or discrete) generative factors of data. In this section, we define setting for VAE model with joint latent distribution of continuous and discrete (categorical) variables. We define $\mathbf{z}$ as continuous part of latent code with prior $p_{\boldsymbol{\theta}}(\mathbf{z})$ and  $\mathbf{c}$ as discrete part with uniform prior. In this setting, the encoder network represents joint posterior approximation $q_\phi(\mathbf{z},\mathbf{c}|\mathbf{x})$, decoder network is $p_{\boldsymbol{\theta}}(\mathbf{x}|\mathbf{z}, \mathbf{c})$ and prior is $p_{\boldsymbol{\theta}}(\mathbf{z}, \mathbf{c})$. Then, the resulting ELBO objective for this variational autoencoder is
\begin{equation}\label{eq:elbo_joint_latent}
    \mathcal{L}(\boldsymbol{\theta},\boldsymbol{\phi})=\mathbb{E}_{q_{\boldsymbol{\phi}}( \mathbf{z},\mathbf{c}|\mathbf{x})} [\log p_{\boldsymbol{\theta}}(\mathbf{x}|\mathbf{z},\mathbf{c})] - D_{\KL}\KLdel{q_{\boldsymbol{\phi}}( \mathbf{z},\mathbf{c}|\mathbf{x})}{p_{\boldsymbol{\theta}}(\mathbf{z},\mathbf{c})}
\end{equation}
By assumption that $\mathbf{c}$ and $\mathbf{z}$ are conditionally and mutually independent, meaning 
\begin{equation}
q_{\boldsymbol{\phi}}( \mathbf{z},\mathbf{c}|\mathbf{x})=q_{\boldsymbol{\phi}}( 
\mathbf{z}|\mathbf{x})q_{\boldsymbol{\phi}}(\mathbf{c}|\mathbf{x}) \ and \ p_{\boldsymbol{\theta}}(\mathbf{z}, \mathbf{c})=p_{\boldsymbol{\theta}}(\mathbf{z})p_{\boldsymbol{\theta}}(\mathbf{c})
\end{equation}
 we can decompose KL divergence term as
\begin{equation}\label{eq:kl_joint_decomposition}
\begin{aligned}
D_{\KL}\KLdel{q_{\boldsymbol{\phi}}( \mathbf{z},\mathbf{c}|\mathbf{x})}{p_{\boldsymbol{\theta}}(\mathbf{z},\mathbf{c})}&=\mathbb{E}_{q_{\boldsymbol{\phi}}( \mathbf{z},\mathbf{c}|\mathbf{x})}\Big[\log\frac{q_{\boldsymbol{\phi}}( \mathbf{z},\mathbf{c}|\mathbf{x})}{p_{\boldsymbol{\theta}}(\mathbf{z},\mathbf{c})}\Big]\\
&=\mathbb{E}_{q_{\boldsymbol{\phi}}( \mathbf{z}|\mathbf{x})}\Big[\mathbb{E}_{q_{\boldsymbol{\phi}}(\mathbf{c}|\mathbf{x})}\Big[\log\frac{q_{\boldsymbol{\phi}}(\mathbf{z}|\mathbf{x})q_{\boldsymbol{\phi}}(\mathbf{c}|\mathbf{x})}{p_{\boldsymbol{\theta}}(\mathbf{z})p_{\boldsymbol{\theta}}(\mathbf{c})}\Big]\Big]\\&=\mathbb{E}_{q_{\boldsymbol{\phi}}( \mathbf{z}|\mathbf{x})}\Big[\mathbb{E}_{q_{\boldsymbol{\phi}}(\mathbf{c}|\mathbf{x})}\Big[\log\frac{q_{\boldsymbol{\phi}}(\mathbf{z}|\mathbf{x})}{p_{\boldsymbol{\theta}}(\mathbf{z})}\Big]\Big] + \mathbb{E}_{q_{\boldsymbol{\phi}}( \mathbf{z}|\mathbf{x})}\Big[\mathbb{E}_{q_{\boldsymbol{\phi}}(\mathbf{c}|\mathbf{x})}\Big[\log\frac{q_{\boldsymbol{\phi}}(\mathbf{c}|\mathbf{x})}{p_{\boldsymbol{\theta}}(\mathbf{c})}\Big]\Big]\\&=\mathbb{E}_{q_{\boldsymbol{\phi}}( \mathbf{z}|\mathbf{x})}\Big[\log\frac{q_{\boldsymbol{\phi}}(\mathbf{z}|\mathbf{x})}{p_{\boldsymbol{\theta}}(\mathbf{z})}\Big] + \mathbb{E}_{q_{\boldsymbol{\phi}}(\mathbf{c}|\mathbf{x})}\Big[\log\frac{q_{\boldsymbol{\phi}}(\mathbf{c}|\mathbf{x})}{p_{\boldsymbol{\theta}}(\mathbf{c})}\Big]\\&=D_{\KL}\KLdel{q_{\boldsymbol{\phi}}( \mathbf{z}|\mathbf{x})}{p_{\boldsymbol{\theta}}(\mathbf{z})}+D_{\KL}\KLdel{q_{\boldsymbol{\phi}}( \mathbf{c}|\mathbf{x})}{p_{\boldsymbol{\theta}}(\mathbf{c})}.
\end{aligned} 
\end{equation}
Then, we can summarize it as
\begin{equation}\label{eq:kl_joint_decomposition_summ}
D_{\KL}\KLdel{q_{\boldsymbol{\phi}}( \mathbf{z},\mathbf{c}|\mathbf{x})}{p_{\boldsymbol{\theta}}(\mathbf{z},\mathbf{c})}=D_{\KL}\KLdel{q_{\boldsymbol{\phi}}( \mathbf{z}|\mathbf{x})}{p_{\boldsymbol{\theta}}(\mathbf{z})}+D_{\KL}\KLdel{q_{\boldsymbol{\phi}}( \mathbf{c}|\mathbf{x})}{p_{\boldsymbol{\theta}}(\mathbf{c})}
\end{equation}
Having this KL-divergence, decomposition we can rewrite ELBO objective in Equation~\ref{eq:elbo_joint_latent} as:
\begin{equation}\label{eq:elbo_joint_latent_final}
    \mathcal{L}(\boldsymbol{\theta},\boldsymbol{\phi})=\mathbb{E}_{q_{\boldsymbol{\phi}}( \mathbf{z},\mathbf{c}|\mathbf{x})} [\log p_{\boldsymbol{\theta}}(\mathbf{x}|\mathbf{z},\mathbf{c})] - D_{\KL}\KLdel{q_{\boldsymbol{\phi}}( \mathbf{z}|\mathbf{x})}{p_{\boldsymbol{\theta}}(\mathbf{z})}-D_{\KL}\KLdel{q_{\boldsymbol{\phi}}( \mathbf{c}|\mathbf{x})}{p_{\boldsymbol{\theta}}(\mathbf{c})}
\end{equation}

\textbf{Continuous latent varibles.} In our experiments, we represent continuous latent variables by Gaussian distributions. Therefore we have approximate posterior distribution for continuous varible defined as $q_{\boldsymbol{\phi}}( \mathbf{z}|\mathbf{x})=\mathcal{N}(\mathbf{z}; \boldsymbol{\mu}(\mathbf{x}), \boldsymbol{\sigma}^2(\mathbf{x})I)$ with prior $p_{\boldsymbol{\theta}}(\mathbf{z})=\mathcal{N}(0,1)$.

\textbf{Discrete latent variables.} To make the approximate posterior distribution over discrete latent variable $q_{\boldsymbol{\phi}}(\mathbf{c}|\mathbf{x})$ differentiable, we employ relaxation techique proposed by \cite{jang2016categorical,maddison2016concrete} that is based on reparametrization Gumbel Max trick \cite{gumbel1954statistical} (which involves non-differentiable argmax opertaion). Let $\mathbf{c}$ be a categorical random variable with $K$ categories and class probabilities $\pi_1\ldots\pi_{K}$. We represent categorical variable as $K$-dimensional one-hot vector. Also, let $g_{1} \ldots g_{K}$ be i.i.d samples drawn from Gumbel$(0,1)$. Then, using following softmax function we can sample continuous approximation of categorical distribution forming k-dimensional vector with each member computed as  
\begin{equation} \label{eq:gumbel-softmax-sample}
    y_i = 
    \frac{\exp{((\log(\pi_i) + g_i) / \tau)}}
    {\sum_{j = 1}^{K} \exp{((\log(\pi_j) + g_j)/ \tau)} }
\end{equation}
for $i = 1, $\dots$, K$. $\tau$ is a temperature coefficient which can be seen as a hyperparameter of the trained model and can be varied during process of training. This relaxed distribution is referred as Concrete or Gumbel Softmax distribution and lets denote is as $GS(\boldsymbol{\pi})$, where $\boldsymbol{\pi}$ defines probabilities of categorical variable. Having this formulation for relaxation of discrete distribution, we now can make $q_{\boldsymbol{\phi}}(\mathbf{c}|\mathbf{x})$ differentiable by defining it as $q_{\boldsymbol{\phi}}(\mathbf{c}|\mathbf{x})=GS(\boldsymbol{\pi}(\mathbf{x}))$. We set prior distribution $p_{\boldsymbol{\theta}}(\mathbf{c})$ of categorical approximate posterior as uniform categorical distribution over $K$ categories. The resulting model is represented in Figure \ref{fig:joint_model}
\begin{figure}[h]
\centering
\includegraphics[width=0.6\linewidth]{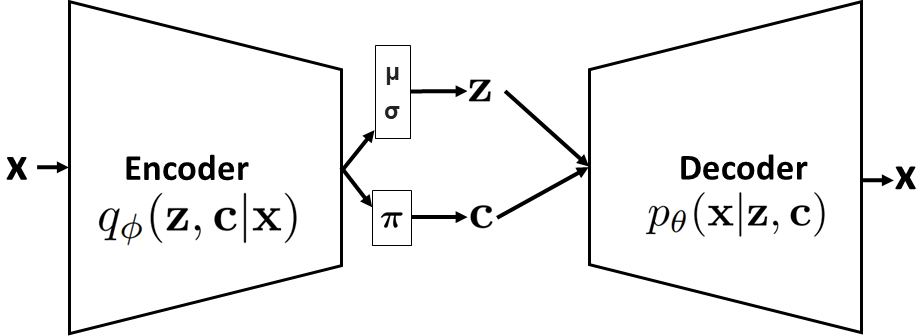}
\caption{Representation of VAE model with joint continuous and discrete latent variable}
\label{fig:joint_model}
\end{figure}

In the experiments on MNIST and FashionMNIST we employ VAE with 16-dimensional Gaussian part of latent varible and one categorical latent variable with $K$=10 categories since this datasets both have 10 classes. For this VAE form trained with mutual information maximization, we maximize MI with respect to observations $\mathbf{x}$ and categorical latent variable $\mathbf{c}$. In that case, the mutual information maximization regularizer term has form of
\begin{equation} \label{eq:mut_information_categ}
    \mathit{MI}(\boldsymbol{\theta}, \boldsymbol{\phi}, Q)=
    \mathbb{E}_{\mathbf{c} \sim q_{\boldsymbol{\phi}}(\mathbf{c}|\mathbf{x}), \mathbf{x} \sim  p_{\boldsymbol{\theta}}(\mathbf{x}|\mathbf{z},\mathbf{c})} [\log Q(\mathbf{c}|\mathbf{x})] + H(\mathbf{c}).
\end{equation}
And the full objective for VAE with joint Gaussian and discrete distribution with information maximization becomes
\begin{equation}\label{eq:elbo_joint_latent_final_mi}
\begin{aligned}
    \mathcal{L}(\boldsymbol{\theta},\boldsymbol{\phi},Q)=\mathbb{E}_{q_{\boldsymbol{\phi}}( \mathbf{z},\mathbf{c}|\mathbf{x})} [\log p_{\boldsymbol{\theta}}(\mathbf{x}|\mathbf{z},\mathbf{c})] - D_{\KL}\KLdel{q_{\boldsymbol{\phi}}( \mathbf{z}|\mathbf{x})}{p_{\boldsymbol{\theta}}(\mathbf{z})}-D_{\KL}\KLdel{q_{\boldsymbol{\phi}}( \mathbf{c}|\mathbf{x})}{p_{\boldsymbol{\theta}}(\mathbf{c})} \\ +  \lambda\mathit{MI}(\boldsymbol{\theta}, \boldsymbol{\phi}, Q)
\end{aligned}
\end{equation}

In this setting, we train and compare two identically initialized VAE models with same hyperparameters. One using only objective Eq. \ref{eq:elbo_joint_latent_final} and other with $MI$ regularizer Eq. \ref{eq:mut_information_categ} for mutual information maximization using aforementioned objective Eq. \ref{eq:elbo_joint_latent_final_mi}.
\section{Experimental Results\label{sec:exp_setting}}
\subsection{VAE with Gaussian Latent Variable}
As we mentioned before, we trained two identically initialized VAE models: one using ELBO objective and one with added $MI$ regularizer for sub-part of latent code ($\mathbf{z}_{1}$,$\mathbf{z}_{2}$)=$\mathbf{\hat{z}}$. In Figure \ref{fig:gaussian_manipulations} we provide a qualitative comparison of the impact of these two components of the code on produced samples. For each latent code, we vary $\mathbf{z}_{1}$ and $\mathbf{z}_{2}$ from -3 to 3 with fixed remaining part and decode it. As you can see on Figure \ref{fig:gaussian_manipulations} (a), $\mathbf{\hat{z}}$ in vanilla VAE does not have much impact on the output samples. In contrast, $\mathbf{\hat{z}}$ with maximized mutual information in VAE by $MI$ regularizer have a significant impact on output samples (Figure \ref{fig:gaussian_manipulations} (b)). For this model, we can see that outputs morph between three digit types as the code changes. Moreover, you can see that the particular combinations of these two components of 32-dimensional code morph the original sample into digit 1 and 6 regardless of the original sample type. All of this means that the provided regularizer indeed forces this part of learned latent codes $\mathbf{\hat{z}}$ to have high MI and strong relationship with observations.
\begin{figure}[h]
	\centering
	\begin{subfigure}[h]{0.45\textwidth}
		\includegraphics[width=\textwidth]{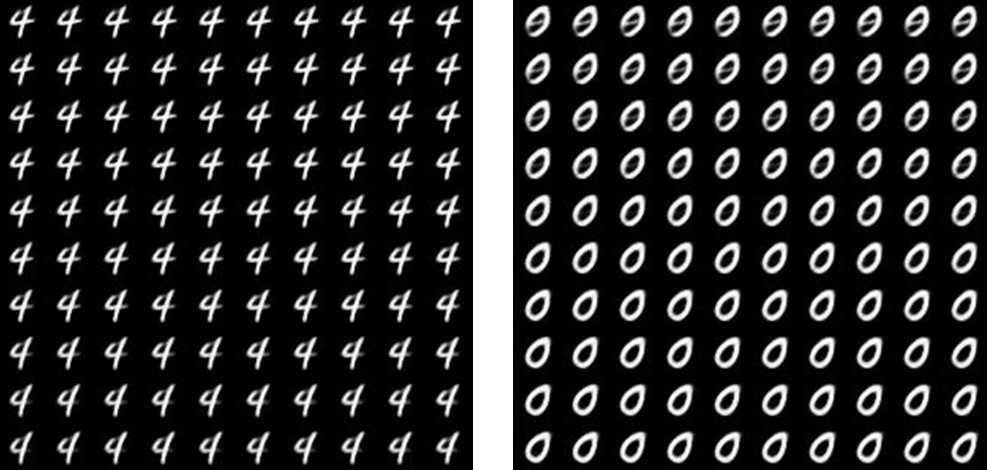}
		\caption{VAE}
	\end{subfigure}
	\begin{subfigure}[h]{0.80\textwidth}
		\includegraphics[width=\textwidth]{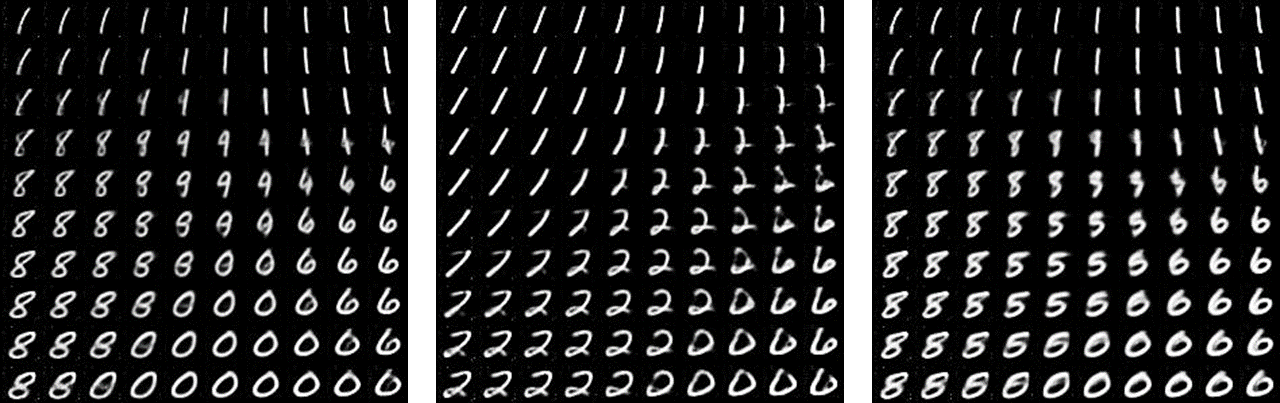}
		\caption{VAE with MI maximization}
	\end{subfigure}
	\caption{Latent code manipulations of samples from VAE with Gaussian latent variable: (a) trained using only ELBO objective, (b) trained with MI maximization. We vary each component of ($\mathbf{z}_{1}$,$\mathbf{z}_{2}$)=$\mathbf{\hat{z}}$ from -3 to 3 having the remaining part of the code fixed. Rows and columns represent values of $\mathbf{z}_{1}$ and $\mathbf{z}_{2}$ respectively}
\label{fig:gaussian_manipulations}
\end{figure}
\begin{figure}[h]
\centering
\includegraphics[width=0.4\linewidth]{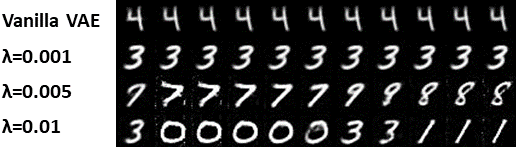}
\caption{Comparison of $\mathbf{z}_{1}$ impact change on the output samples with different values of $MI$ coefficient $\lambda$. The first column represents the original sample and each row represents varied $\mathbf{z}_{1}$ from -5 to 5.}
\label{fig:lambda_impact}
\end{figure}

Also, we compare resulting models with different values of scaling coefficient of $MI$ regularizer $\lambda$ in Figure \ref{fig:lambda_impact}. As you can see, with a low value of lambda, the impact of $\mathbf{z}_{1}$ is the same as in vanilla VAE (trained using only ELBO). However, with the increase of $\lambda$, the impact of $\mathbf{z}_{1}$ on observations also increases. The rest of the latent code is fixed while varying $\mathbf{z}_{1}$.

\subsection{VAE with joint Gaussian and Discrete Latent Variable}
In this section, we compare two identically initialized VAE models with joint Gaussian and discrete latent variables trained on MNIST and FashionMNIST datasets but trained in a different manner. One model is trained using only the ELBO of the form represented by Equation \ref{eq:elbo_joint_latent}. The second one is trained with added $MI$ regularizer (Eq. \ref{eq:elbo_joint_latent_final_mi}) for MI maximization between data samples and categorical part of the learned latent code.

As you can see on Fig.\ref{fig:categorical_manipulations} (a), in VAE that was trained using pure ELBO, the categorical part of the latent code does not have an influence on produced samples. Thus, even when our strong prior assumption that the data have 10 categories was incorporated into the latent variable, the trained model ignores it and does not assign any interpretable representation to the categorical variable.
\begin{figure}[h]
	\centering
	\begin{subfigure}[t]{0.45\textwidth}
		\includegraphics[width=\textwidth]{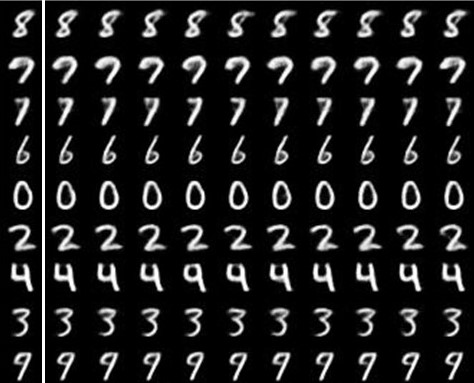}
		\caption{VAE}
	\end{subfigure}
	~
	\begin{subfigure}[t]{0.45\textwidth}
		\includegraphics[width=\textwidth]{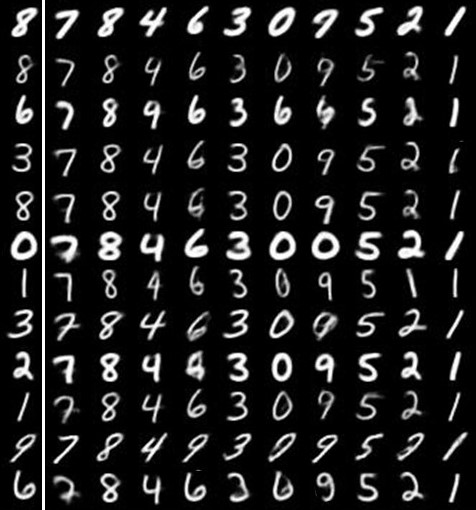}
		\caption{VAE with MI maximization}
	\end{subfigure}
	\caption{Latent code manipulations of samples from VAE with joint Gaussian and discrete latent variabl trained on MNIST with only ELBO objective (a) and with $MI$ regularizer (b). The first (separated) column represents the original samples. The following rows represent this samples with changed categorical part of latent code between 10 categories with fixed Gaussian component. Also, each (not separated) column can be seen as generated samples that are conditioned on a particular category with varied Gaussian component.}
\label{fig:categorical_manipulations}
\end{figure}

In contrast, for VAE that was trained with MI maximization between observations and categorical code, produced samples show a completely different response to the latent categorical variable change. As the categorical part of latent code varies between 10 categories, the samples change in a class-wise manner. For most of the samples, the particular value of the categorical variable changes them to the same digit type while preserving other features of original sample like thickness and angle. We interpret it as that the model generalizes and disentangles digit type from style representations by categorical and Gaussian part of the latent code respectively.

For the sake of quantitative comparison, we applied the encoder categorical component as a classifier to the MNIST classification task. VAE trained with ELBO objective has 21\% classification accuracy on MNIST while VAE trained with $MI$ regularizer achieved 82\% accuracy.

We have performed same experiments on VAE models trained on FashionMNIST dataset and the results are similar to those that was obtained from models trained on MNIST dataset. VAE model trained with pure ELBO objective does not assign almost any representation to the discrete random variable. The model trained with MI maximization for this variable incorporate into it more representation power. With the change of categorical variable the category of sample also changes. You can see comparison of categorical variable change impact on produced samples from both models on Figure \ref{fig:categorical_manipulations_fashion}.

\begin{figure}[h]
	\centering
	\begin{subfigure}[t]{0.45\textwidth}
		\includegraphics[width=\textwidth]{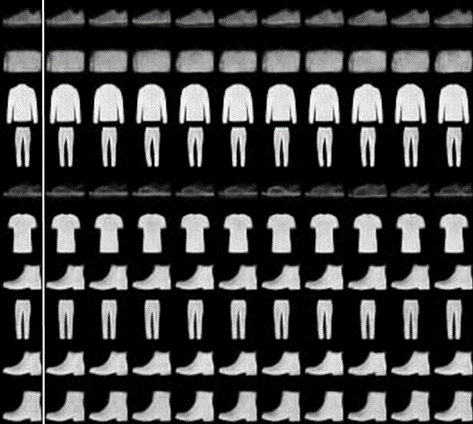}
		\caption{VAE}
	\end{subfigure}
	~
	\begin{subfigure}[t]{0.45\textwidth}
		\includegraphics[width=\textwidth]{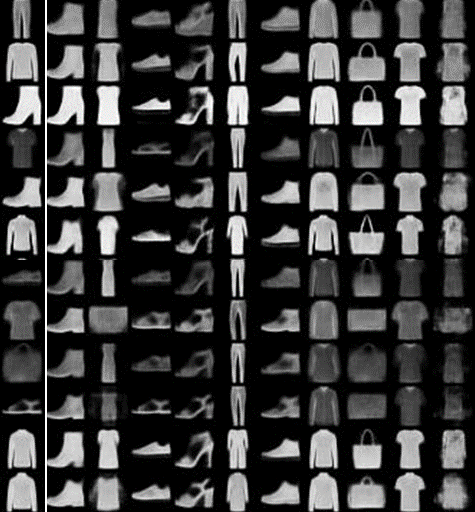}
		\caption{VAE with MI maximization}
	\end{subfigure}
	\caption{Latent code manipulations of samples from VAE with joint Gaussian and discrete latent variabl trained on FashionMNIST with only ELBO objective (a) and with $MI$ regularizer (b). The first (separated) column represents the original samples. The following rows represent this samples with changed categorical part of latent code between 10 categories with fixed Gaussian component. Also, each (not separated) column can be seen as generated samples that are conditioned on a particular category with varied Gaussian component.}
\label{fig:categorical_manipulations_fashion}
\end{figure}

In Fig.\ref{fig:probability_histograms} we compare histograms of categorical latent variable probabilities that were collected during the training of both models. As you can see, for the case of VAE trained using ELBO objective, the probabilities are mostly concentrated around 0.1 and do not reach the area around 1. In the case of VAE with maximized MI for the categorical variable, the probability values are concentrated around 0.5 and 0.9. It is natural behavior since when the category probabilities are uniform regardless of the input, it is pointless to do any further inference using this variable for the decoder network and thus the resulting model ignores it. Therefore, we interpret this observation as that VAE model with maximized MI between data samples and categorical part of the latent variable indeed makes more use of this variable by redistributing category probabilities.
\begin{figure}[h!]
	\centering
	\begin{subfigure}[t]{0.50\textwidth}
		\includegraphics[width=\textwidth]{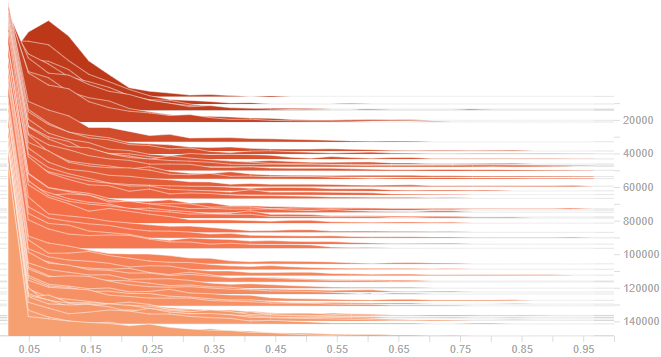}
		\caption{VAE}
	\end{subfigure}
	~
	\begin{subfigure}[t]{0.54\textwidth}
		\includegraphics[width=\textwidth]{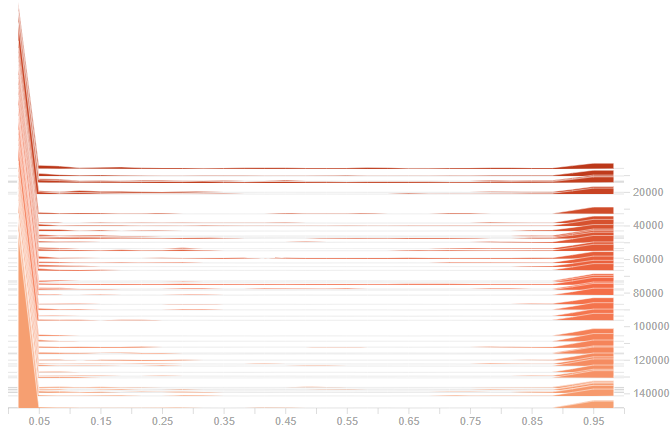}
		\caption{VAE with MI maximization}
	\end{subfigure}
	\caption{Histograms collected from categorical latent variable probabilities of VAE with joint Gaussian and discrete latent variable during process of training.}
\label{fig:probability_histograms}
\end{figure}
\\~\\

As we mentioned before, our proposed Variational Mutual Information Maximization Framework can be used for MI evaluation between latent variables and observations for a fixed VAE by obtaining lower bound on MI (Eq.\ref{eq:mut_information_bound}). In Figure \ref{fig:mi_plots}, we provide plots of lower bound MI estimate between observations and categorical part of latent codes during the training process of two models. One was trained using only the ELBO and the other with MI maximization. As you can see, the VAE model with MI maximization has higher MI lower bound estimate during training than one that was trained without $MI$ regularizer which is an intuitive result.

From the theoretical perspective on KL divergence estimate, value of this estimate is an upper bound on MI between latent variables and observations when taken in expectation over the data \cite{hoffman2016elbo,makhzani2017pixelgan,kim2018disentangling,burgess2018understanding}. Meaning
\begin{equation}\label{eq:elbo_joint_latent_final_mi}
\begin{aligned}
    \mathbb{E}_{p(\mathbf{x})} [D_{\KL}\KLdel{q_{\boldsymbol{\phi}}( \mathbf{z}|\mathbf{x})}{p(\mathbf{z})}]= I(\mathbf{z};\mathbf{x}) + D_{\KL}\KLdel{q(\mathbf{z})}{p(\mathbf{z})} \geq I(\mathbf{z};\mathbf{x}).
\end{aligned}
\end{equation}
Our experimental results are consistent with it. We provide categorical distribution KL-divergence estimate plots that were collected during training of models with and without MI maximization in Figure \ref{fig:kl_plots}. The model with maximized MI between categorical latent variable and observations has higher KL-divergence estimate during training than one that was trained using only ELBO. Moreover, it is close to the maximum possible value of mutual information for this categorical variable which has discrete uniform prior with 10 categories. The maximum possible value is entropy of this uniform distribution which is $H(\mathbf{c}) \approx 2.3$ and our KL-divergence estimate is close to it having 2.27.

\begin{figure}[h!]
\centering
\includegraphics[width=0.6\linewidth]{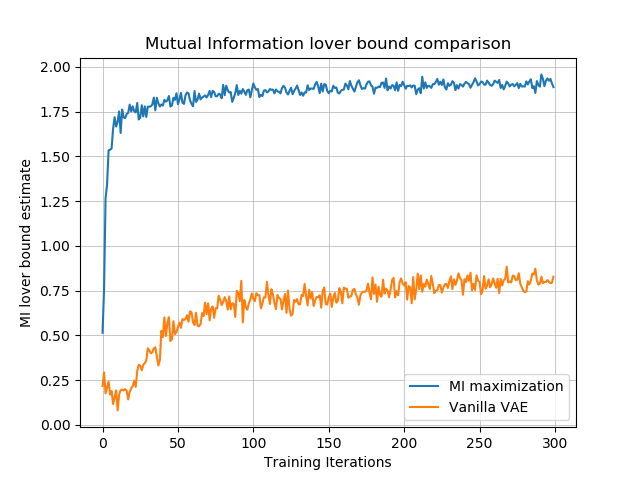}
\caption{MI lower bound estimate for categorical latent variable during the process training VAE models (joint Gaussian and discrete latent) without MI maximization and with $MI$ regularizer.}
\label{fig:mi_plots}
\end{figure}

\begin{figure}[h!]
\centering
\includegraphics[width=0.6\linewidth]{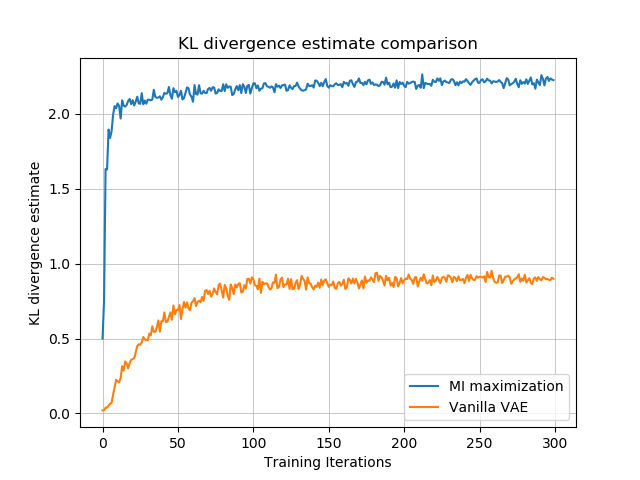}
\caption{KL divergence estimate for categorical latent variable during the process of training VAE models (joint Gaussian and discrete latent) without MI maximization and with $MI$ regularizer.}
\label{fig:kl_plots}
\end{figure}

We have counted numbers of particular digits from MNIST dataset encoded into particular one-hot vectors (categorical variables) for VAE models trained with and without MI maximization. We represent this results in figures \ref{fig:hists_digits_mi} and \ref{fig:hists_digits_nomi}. As you can see, the digit images from particular classes align well with particular one-hot vectors in case of VAE with MI maximization. In contrast, for VAE without MI maximization, the distribution of particular type digit images are almost uniform across all one-hot vectors.
\newpage
\begin{figure}[h]
	\centering
	\begin{subfigure}[t]{0.3\textwidth}
		\includegraphics[width=\textwidth]{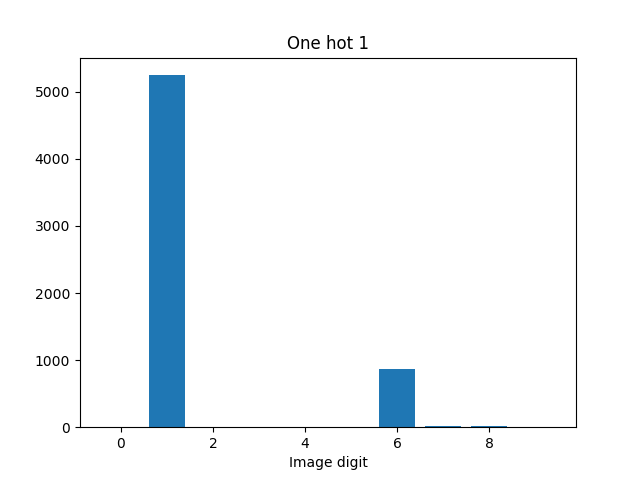}
		\caption{One-hot 1}
	\end{subfigure}
	\begin{subfigure}[t]{0.3\textwidth}
		\includegraphics[width=\textwidth]{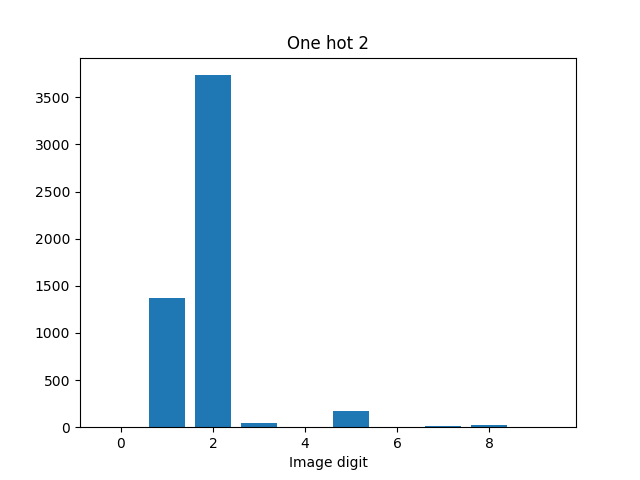}
		\caption{One-hot 2}
	\end{subfigure}
		\begin{subfigure}[t]{0.3\textwidth}
		\includegraphics[width=\textwidth]{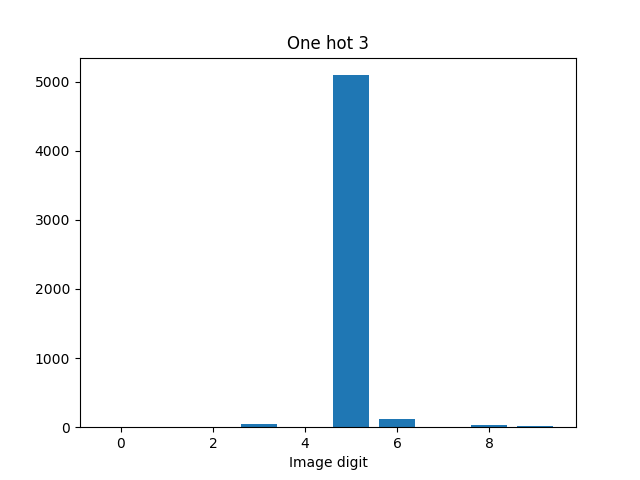}
		\caption{One-hot 3}
	\end{subfigure}
		\begin{subfigure}[t]{0.3\textwidth}
		\includegraphics[width=\textwidth]{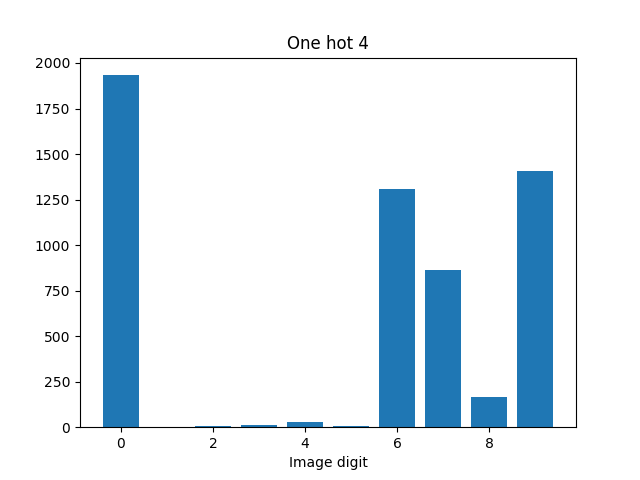}
		\caption{One-hot 4}
	\end{subfigure}
		\begin{subfigure}[t]{0.3\textwidth}
		\includegraphics[width=\textwidth]{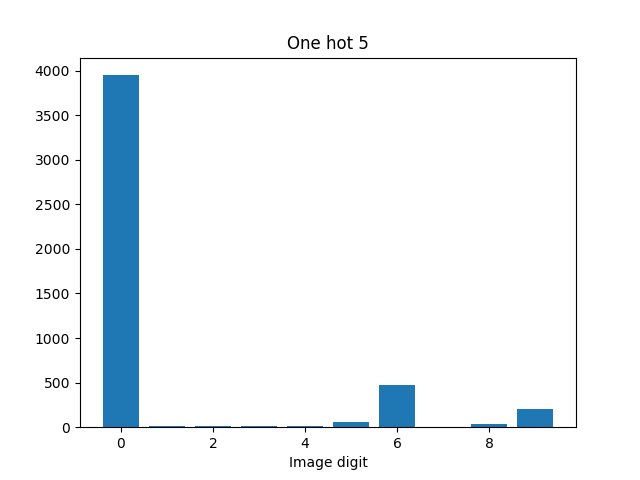}
		\caption{One-hot 5}
	\end{subfigure}
		\begin{subfigure}[t]{0.3\textwidth}
		\includegraphics[width=\textwidth]{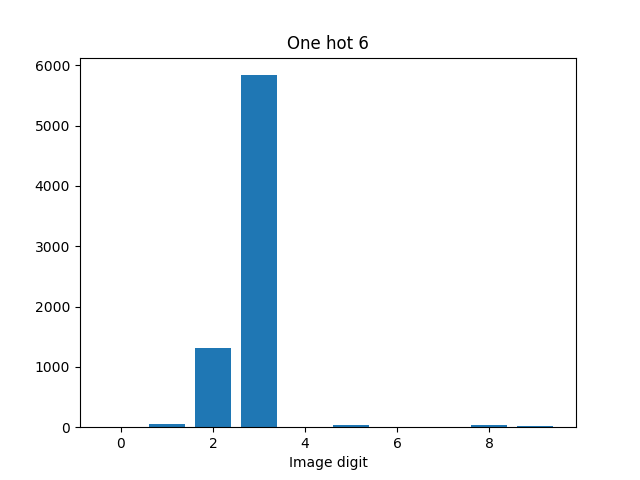}
		\caption{One-hot 6}
	\end{subfigure}
		\begin{subfigure}[t]{0.3\textwidth}
		\includegraphics[width=\textwidth]{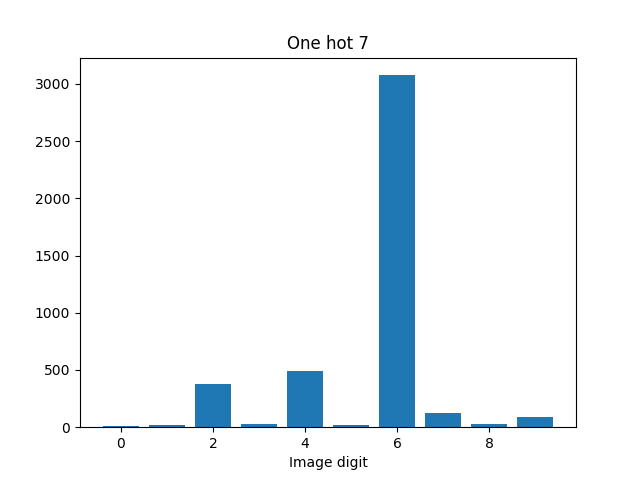}
		\caption{One-hot 7}
	\end{subfigure}
		\begin{subfigure}[t]{0.3\textwidth}
		\includegraphics[width=\textwidth]{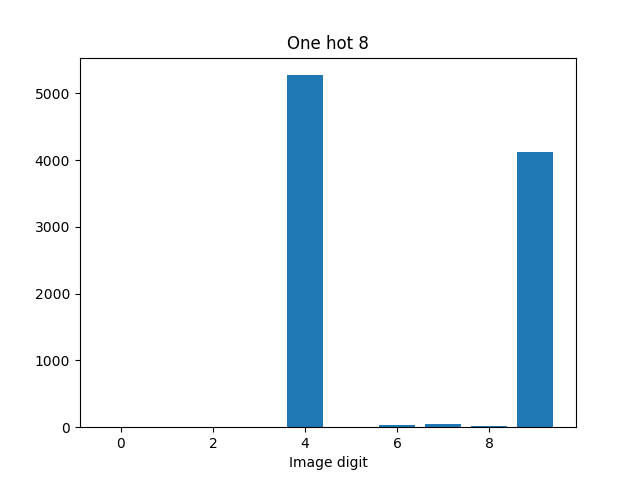}
		\caption{One-hot 8}
	\end{subfigure}
		\begin{subfigure}[t]{0.3\textwidth}
		\includegraphics[width=\textwidth]{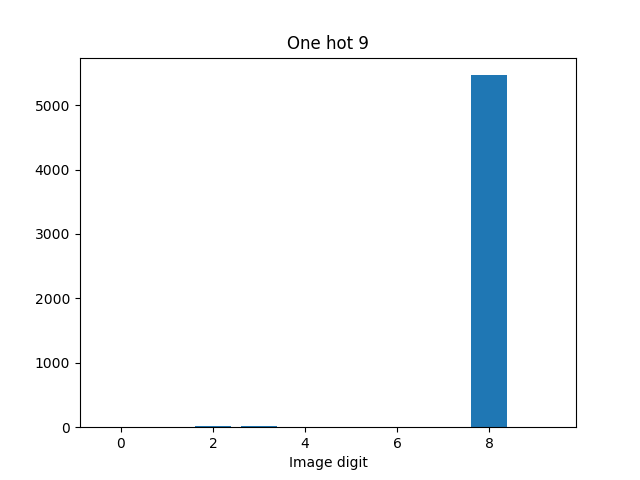}
		\caption{One-hot 9}
	\end{subfigure}
		\begin{subfigure}[t]{0.3\textwidth}
		\includegraphics[width=\textwidth]{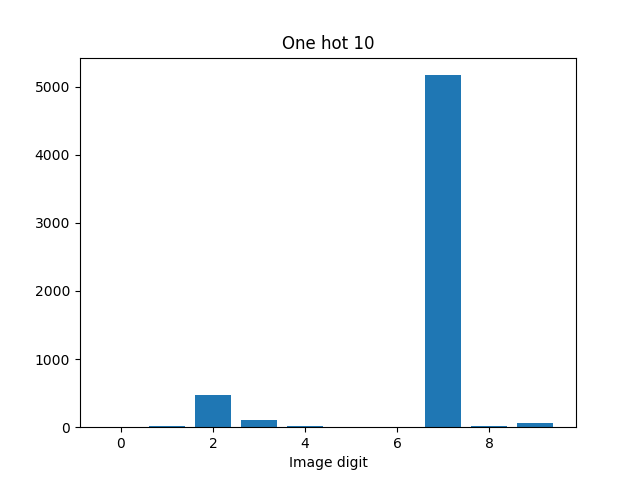}
		\caption{One-hot 10}
	\end{subfigure}
	\caption{Histograms of numbers of digits encoded into one-hot vectors (that represent categorical latent variable). VAE with MI maximization.}
\label{fig:hists_digits_mi}
\end{figure}
\newpage
\begin{figure}[h]
	\centering
	\begin{subfigure}[t]{0.3\textwidth}
		\includegraphics[width=\textwidth]{onehot_histograms_mi/onehot_hist_0.png}
		\caption{One-hot 1}
	\end{subfigure}
	\begin{subfigure}[t]{0.3\textwidth}
		\includegraphics[width=\textwidth]{onehot_histograms_mi/onehot_hist_1.png}
		\caption{One-hot 2}
	\end{subfigure}
	\begin{subfigure}[t]{0.3\textwidth}
		\includegraphics[width=\textwidth]{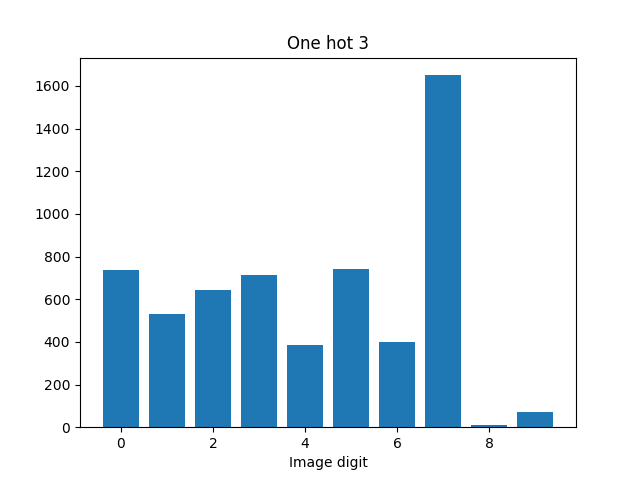}
		\caption{One-hot 3}
	\end{subfigure}
	\begin{subfigure}[t]{0.3\textwidth}
		\includegraphics[width=\textwidth]{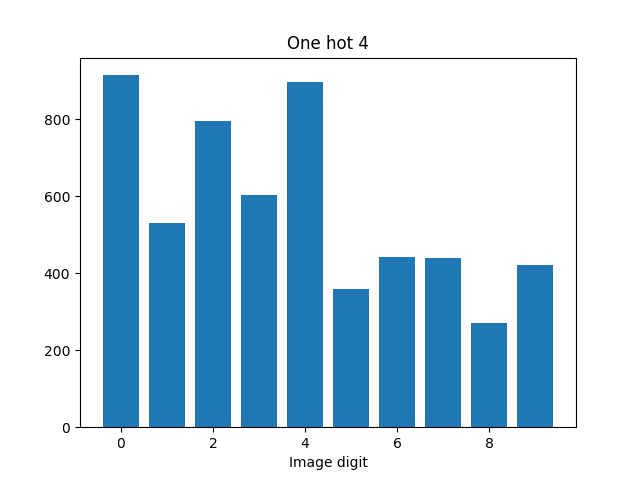}
		\caption{One-hot 4}
	\end{subfigure}
	\begin{subfigure}[t]{0.3\textwidth}
		\includegraphics[width=\textwidth]{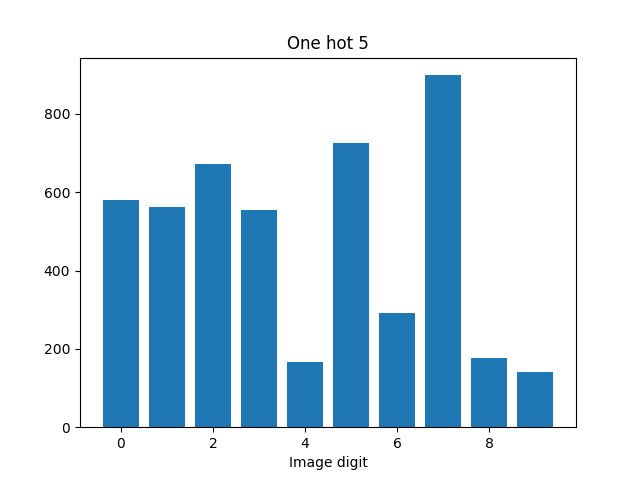}
		\caption{One-hot 5}
	\end{subfigure}
	\begin{subfigure}[t]{0.3\textwidth}
		\includegraphics[width=\textwidth]{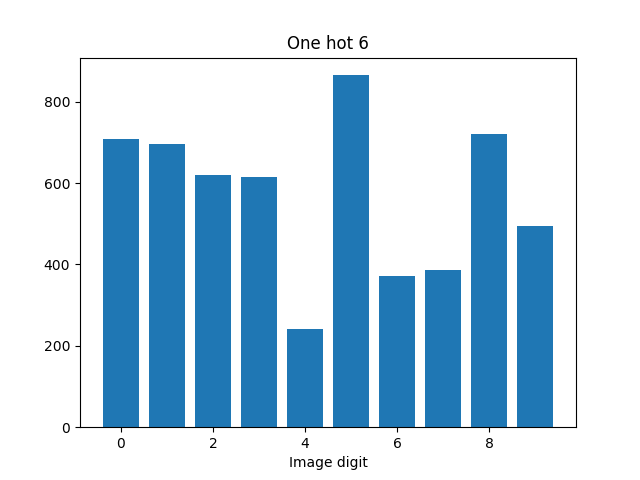}
		\caption{One-hot 6}
	\end{subfigure}
	\begin{subfigure}[t]{0.3\textwidth}
		\includegraphics[width=\textwidth]{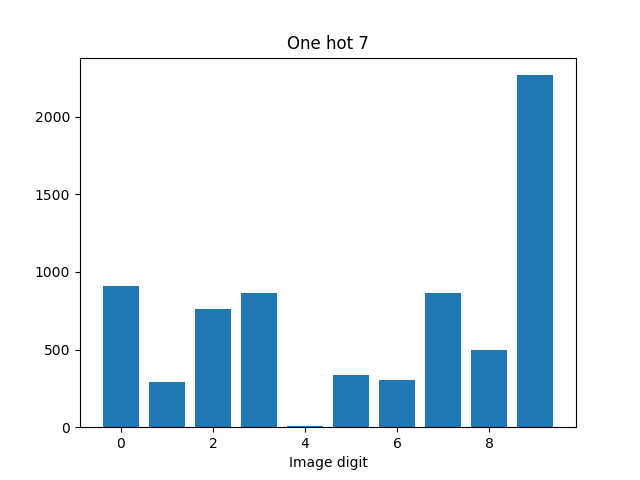}
		\caption{One-hot 7}
	\end{subfigure}
	\begin{subfigure}[t]{0.3\textwidth}
		\includegraphics[width=\textwidth]{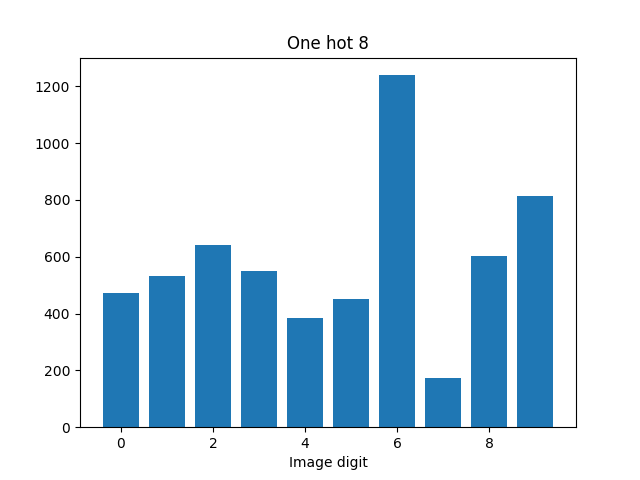}
		\caption{One-hot 8}
	\end{subfigure}
	\begin{subfigure}[t]{0.3\textwidth}
		\includegraphics[width=\textwidth]{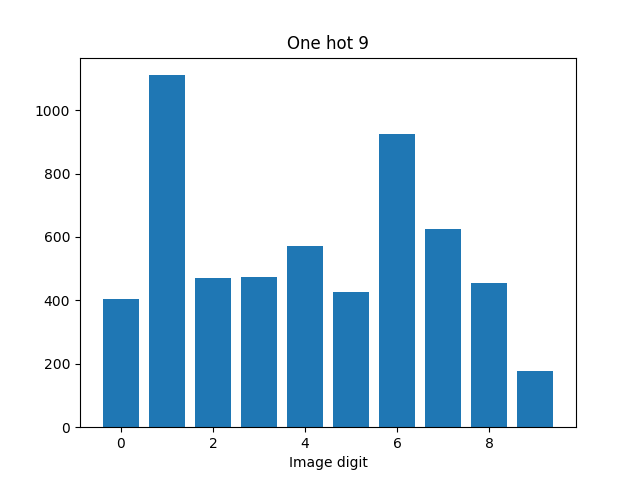}
		\caption{One-hot 9}
	\end{subfigure}
	\begin{subfigure}[t]{0.3\textwidth}
		\includegraphics[width=\textwidth]{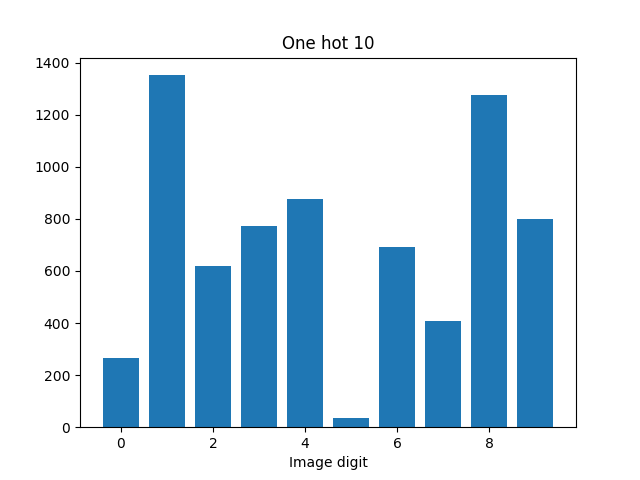}
		\caption{One-hot 10}
	\end{subfigure}
	\caption{Histograms of numbers of digits encoded into one-hot vectors (that represent categorical latent variable). VAE without MI maximization.}
\label{fig:hists_digits_nomi}
\end{figure}

\newpage\section{Intuition\label{sec:intuition}}
The $Q$ network that represents auxiliary distribution can be seen as a classifier network or a feature extractor network. When we maximize $MI$ regularizer
\begin{equation} \label{eq:mut_information_reg_again}
    \mathit{MI}(\theta, \phi, Q) = \mathbb{E}_{\mathbf{z} \sim q_{\boldsymbol{\phi}}(\mathbf{z}|\mathbf{x}), \mathbf{x} \sim  p_{\boldsymbol{\theta}}(\mathbf{x}|\mathbf{z})} [\log Q(\mathbf{z}|\mathbf{x})] + H(\mathbf{z})
\end{equation}
with respect to parameters of this network, the maximization of the first term is the same as minimization of negative log-likelihood as we do when train classification models. Thus, we can interpret the whole training procedure as training the Q network to correctly classify $\mathbf{x}$ in terms of its original generative factors $\mathbf{z}$ or to better extract them from the given sample. Then, when maximizing $MI$ w.r.t. VAE, we are forcing the model to make these features more extractable from the produced samples and more classifiable.
\chapter{Related Works}
\section{FactorVAE\label{sec:factavae}}
In \cite{kim2018disentangling} authors proposed FactorVAE which is a method for learning disentangled representations for VAE \cite{kingma2013auto,rezende2014stochastic} framework on data that is generated from independent variation factors. The key concept of the method is marginal distribution of approximate posterior which authors interpret as distribution of representations of the entire dataset or aggregated posterior \cite{makhzani2015adversarial}. It is defined as follows
\begin{equation}
q(z) = \mathbb{E}_{p_{data}(\mathbf{x})}[q(\mathbf{z}|\mathbf{x})] = \frac{1}{N} \sum_{i=1}^N q(\mathbf{z}|\mathbf{x}^{(i)}).
\end{equation}
Authors argue that as long as we want to vary this factors independently and have disentangled representations, it is desirible to have $q(\mathbf{z}) = \prod_{j=1}^d q(\mathbf{z}_j)$. The stepping stone to the final idea is theoretical insight that KL divergence term in VAE objective can be decomposed \cite{hoffman2016elbo,makhzani2017pixelgan,kim2018disentangling,burgess2018understanding} as
\begin{equation}
\begin{aligned}
    \mathbb{E}_{p_{data}(\mathbf{x})} [D_{\KL}\KLdel{q( \mathbf{z}|\mathbf{x})}{p(\mathbf{z})}]= I(\mathbf{z};\mathbf{x}) + D_{\KL}\KLdel{q(\mathbf{z})}{p(\mathbf{z})}
\end{aligned}
\end{equation}
meaning that penalizing the $D_{\KL}\KLdel{q(\mathbf{z})}{p(\mathbf{z})}$ term improves disentanglment but penalising $I(\mathbf{z};\mathbf{x})$ term will lead for poor quality reconstructions as for high values of $\beta$ in $\beta$-VAE. Thus, authors introduce an additional term for original VAE ELBO objective that will explicitly encourage independence of the latent code components. The resulting objective for FactorVAE is
\begin{equation} \label{eq:factvae_loss}
\frac{1}{N}\sum_{i=1}^N  \Big[ \mathbb{E}_{q(\mathbf{z}|\mathbf{x}^{(i)})}[\log p(\mathbf{x}^{(i)}|\mathbf{z})] - D_{\KL}\KLdel{q(\mathbf{z}|\mathbf{x}^{(i)})}{p(\mathbf{z})} \Big] 
- \gamma D_{\KL}\KLdel{q(\mathbf{z})}{\bar{q}(\mathbf{z})},
\end{equation}
where $\bar{q}(\mathbf{z}) \coloneqq \prod_{j=1}^d q(\mathbf{z}_j)$. The introduced term is referred as Total Correlation \cite{Watanabe:1960:ITA:1661258.1661265} which is a measure of dependence for random variables. Since the introduced term is intractable, authors employ additional discriminator network to perform density-ratio trick \cite{nguyen2010estimating,sugiyama2012density}. Having a discriminator and VAE, authors train them jointly. The discriminator is trained to distinguish between $\bar{q}(\mathbf{z})$ and $q(\mathbf{z})$ and VAE is trained using aforementioned objective.
\section{$\beta$-TCVAE\label{sec:betatcvae}}
In \cite{chen2018isolating} authors use decomposition of ELBO that illustrates the existence of a total correlation term between latent distributions. The used decomposition provides insights to the independence of latent codes, and mutual information between latent codes and observations. The used decomposition of ELBO follows \cite{hoffman2016elbo} and authors associate each observed data sample with integer index and define a uniform random variable on $\{ 1, 2, ..., N\}$. Then, authors define $q(\mathbf{z}|\mathbf{n}) = q(\mathbf{z}|\mathbf{x}_n)$ and $q(\mathbf{z},\mathbf{n}) = q(\mathbf{z}|\mathbf{n})p(\mathbf{n}) = q(\mathbf{z}|\mathbf{n})\frac{1}{N}$. Aggregated posterior is defined as $q(\mathbf{z})=\sum_{n=1}^N q(\mathbf{z}|\mathbf{n})p(\mathbf{n})$ \cite{makhzani2015adversarial}. The resulting decomposition relies on KL divergence term of ELBO objective in expectation of data resulting in
\begin{equation}\label{eq:elbo_kl_decomposition}
\begin{split}
\mathbb{E}_{(\mathbf{x})} [D_{\KL}\KLdel{q( \mathbf{z}|\mathbf{n})}{p(\mathbf{z})}]
= \underbrace{D_{\KL}\KLdel{q( \mathbf{z},\mathbf{n})}{q(\mathbf{z})p(\mathbf{n})}}_{\text{i Index-Code MI}} + \underbrace{D_{\KL}\KLdel{q(\mathbf{z})}{\prod_j q(z_j)}}_{\text{ii Total Correlation}} + \underbrace{\sum_j D_{\KL}\KLdel{q(z_j)}{p(z_j)}}_{\text{iii Dimension-wise KL}}.
\end{split}
\end{equation}
The main hypothesis and reasoning is similar to \cite{kim2018disentangling}. Low total correlation \cite{Watanabe:1960:ITA:1661258.1661265} is a key to disentangled representations which is the reason why $\beta$-VAE is capable of learning them by penalizing general KL-divergence term of original ELBO form. Hovewer, it is also penalizes index-code mutual information which may lead to discarding of the necessary information from the latent codes. Authors emphasize that the penalty on total correlation should force the model to find statistically independent factors in the data distribution. Using this theoretical insights, authors propose a modification to the original $\beta$-VAE called $\beta$-TCVAE with objective equivalent to FactorVAE \cite{kim2018disentangling} but optimized in a different manner. The resulting objective of $\beta$-TCVAE is
\begin{align} \label{eq:beta-tc-vae}
\mathcal{L}_{\beta-\text{TC}} :=
\mathbb{E}_{q(\mathbf{z}|\mathbf{n})p(\mathbf{n})}[\log p(\mathbf{n}|\mathbf{z})] - \alpha I_q(\mathbf{z};\mathbf{n}) 
- \beta D_{\KL}\KLdel{q(\mathbf{z})}{\prod_j q(z_j)} -
\gamma\sum_j D_{\KL}\KLdel{q(z_j)}{p(z_j)}
\end{align}
In the experiments, authors discovered that fixing $\alpha=\gamma=1$ and varying $\beta$ leads to better results. In this setting, this objective is same as in FactorVAE when $\gamma=\beta - 1$. The difference is that Kim $\&$ Mnih \cite{kim2018disentangling} estimate total correlation using additional discriminator network. In contrast, authors of $\beta$-TCVAE use minibatch-weighted sampling.
\section{DIP-VAE\label{sec:dipvae}}
In \cite{kumar2017variational} authors propose to add a regularizer to the VAE ELBO objective that will encourage inferred prior or expected variational posterior $
q_{\phi}(\mathbf{z}) = \int q_\phi(\mathbf{z}|\mathbf{x}) p(\mathbf{x}) d\mathbf{x}$ to match the true prior $p(\mathbf{z})$. Authors argue that minimizing KL divergence between this two distributions or any other divergence metric will lead to better and disentangled learned latent representations. This is so since the gap between $D_{\KL}\KLdel{q_{\phi}(\mathbf{z})}{p(\mathbf{z})}$ and $\mathbb{E}_{p(\mathbf{x})}D_{\KL}\KLdel{q_{\phi}(\mathbf{z}|\mathbf{x})}{p_{\theta}(\mathbf{z}|\mathbf{x})}$ may be large when training VAE model reaches some stationary point. The resulting objective of proposed DIP-VAE  (for Disentangled Inferred Prior)  is
\begin{equation}
\max_{\theta,\phi} \ \mathbb{E}_{p(\mathbf{x})} \big[\mathbb{E}_{q_{\boldsymbol{\phi}}( \mathbf{z}|\mathbf{x})} [\log p_{\theta}(\mathbf{x}|\mathbf{z})] - D_{\KL}\KLdel{q_{\phi}( \mathbf{z}|\mathbf{x})}{p(\mathbf{z})} \big] - \lambda D(q_{\phi}(\mathbf{z}) \ \vert\vert \ p(\mathbf{z})),
\end{equation}
where $D$ is an arbitrary distance metric and $\lambda$ controls impact of the added regularizer to the full objective. It is intractable to compute this distance in case of KL-divergence, thus authors employ matching moments technique in their experiments.
\section{Adversarial Autoencoders\label{sec:adv_autoencoders}}
Adversarial Autoencoders \cite{makhzani2015adversarial} introduce an alternative approach for learning directed generative latent variable models. The adversarial autoencoder is an autoencoder that is regularized by matching the the aggregated posterior to an arbitrary prior distribution. In the proposed framework the adversarial network discriminates between latent codes from the approximate posterior distribution and selected arbitrary prior distribution. It guides approximate posterior distribution to match the prior distribution. The encoder tries to full the discriminator network by aggregated posterior distribution that it is true prior distribution.

Adversarial part and the autoencoder are trained jointly in turns: reconstruction phase and the regularization phase. During reconstruction phase, the autoencoder minimizes the reconstruction error. During regularization phase, the adversarial part first updates discriminator to distinguish samples from prior and posterior. After that, the generator (encoder) is updated to full discriminator to bring its distribution closer to the prior.
\section{Wasserstein Autoencoders\label{sec:wass_autoencoders}}
Wasserstein Autoencoders (WAE) proposed in \cite{tolstikhin2017wasserstein} take a different perspective on training latent variable models with insights arising from Optimal Transport theory. In contrast to conventional VAE, WAE aims to minimize any transport distance between unknown true data distribution $P_X$ and model distribution $P_G$ using insights from \cite{bousquet1705optimal}. Having any metric function between two images $c(x,x^{'})$, WAE minimize the objective of the form
\begin{equation}
\min_{Q(Z|X)} \ \mathbb{E}_{P_X} \ \mathbb{E}_{Q(Z|X)} \big[ c(X,G(Z)\big] - \lambda D_Z(Q_Z,P_Z),
\end{equation}
with respect to the parameters of decoder $P_G(X|Z)$.
In the used objective, $Q(Z|X)$ is encoder, $Q_Z(Z)$ is aggregated posterior distribution, $D_Z$ is any divergence metric between distribution over random variable $Z$, and $\lambda$ is positive regularization coefficient. The encoder $Q(Z|X)$ and decoder $G(Z)$ are represented by neural networks in the proposed framework. The proposed objective is similar to the ELBO objective of VAE since the first term is reconstruction term that aligns the encoder-decoder pair that encoded images will be accurately decoded. The second term is matching aggregated posterior $Q_Z$ to the prior distribution $P_Z$ in comparison to the VAE where the point-wise posteriors $Q(Z|X)$ is encouraged to match prior for all observed data points. Authors state that WAE provides explicit control over the shape of the entire encoded dataset distribution while VAE aims to control each point separately. Autors also show that the proposed framework generalizes adversarial autoencoders \cite{makhzani2015adversarial}.

For WAE authors propose to use two different kinds of regularizer $D_Z$: GAN-based and MMD-based. GAN-based option employs Jensen-Shannon divergence and uses adversarial training to estimate it. MMD-based regularizer employs maximum mean discrepancy and it generalizes the proposed InfoVAE model in \cite{zhao2017infovae}.
\section{Discussion\label{sec:rel_discussion}}
In this chapter were covered all major recently published works that are related to improvement of original VAE framework and learned latent representations. FactorVAE (Section \ref{sec:factavae}) and $\beta$-TCVAE (Section \ref{sec:betatcvae}) encourage disentanglement of learned representations by increasing penalty of total correlation term between aggregated posterior and factorial latent distribution. In DIP-VAE (Section \ref{sec:dipvae}) the model employs additional regularizer term to match aggregated posterior and priror distribution by arbitrary distance metric. In this works, mutual information component of the decomposed KL-divergence term is not penalized or controlled in any way. Adversarial Autoencoders and Wasserstein Autoencoders basically introduce alternative way for measuring discrepancy between prior and approximate posterior distribution in autoencoder setting.

In contrast, in our work we aim to increase mutual information between latent codes and observations using explicit regularizer based on variational mutual information lower bound estimate. In the number of works such as \cite{higgins2017beta} and \cite{chen2016infogan} it was argued that increasing mutual information can also lead to better learned representations as we also show in our work. On top of that, the disentanglement of learned latent representations is not the only purpose of our proposed framework. The other purpose of the proposed mutual information regularizer is to force resulting VAE model to not ignore the latent codes and strengthen relationship between them and observations. This issue is not discussed and addressed in any of the works covered in this chapter. 
\chapter{Conclusion}
In our work, we have presented a method for evaluation, control, and maximization of mutual information between latent variables and observations in VAE using variational lower bounding technique on mutual information. In comparison to other related works, it provides an explicit and tractable objective without interferring with original ELBO objective. Using the proposed technique, it is possible to compare MI for different fixed VAE networks. Moreover, our experimental results illustrate that the Variational Mutual Information Maximization Framework can indeed strengthen the relationship between latent variables and observations. Also, it improves learned representations for tasks when original ELBO ojective fail to learn useful representations by latent variable. However, it comes with an increase in computational and memory cost, since mutual information lower bound estimate requires auxiliary distribution $Q$ that we represent by an additional encoder neural network and train. 

We believe, that our work (with further analysis and improvements) have the potential to fill the gaps between previous theoretical insights for VAE from Information Theory perspective since ours empirical results are consistent with them. For instance, KL-divergence term in VAE, by analysis from \cite{hoffman2016elbo,makhzani2017pixelgan,kim2018disentangling,burgess2018understanding}, is an upper bound on true mutual information between latent codes and observations. Our empirical results are consistent with this insights.

With regard to increased computational cost, our method can be potentially combined with various models. For instance, there are number of approaches that combine VAE and GAN and our auxiliary network $Q$ can be applied as small sub-branch of the generator network (in similar manner as in \cite{chen2016infogan}) that will not increase computational and memory cost dramatically but may lead to better quality of learned by model representations. Also, it may lead to faster convergence of the model which we see as potential future research direction.

Finally, one of the future research directions is to use encoder network to represent auxiliary distribution $Q$ as well as approximate posterior simultaneously.

\bibliographystyle{abbrv}
\bibliography{./main}

\begin{thebibliography}{10}

\bibitem{alemi2018fixing}
A.~Alemi, B.~Poole, I.~Fischer, J.~Dillon, R.~A. Saurous, and K.~Murphy.
\newblock Fixing a broken elbo.
\newblock In {\em International Conference on Machine Learning}, pages
  159--168, 2018.

\bibitem{barber2003algorithm}
D.~Barber and F.~Agakov.
\newblock The im algorithm: A variational approach to information maximization.
\newblock In {\em NIPS}, pages 201--208. MIT Press, 2003.

\bibitem{bengio2013representation}
Y.~Bengio, A.~Courville, and P.~Vincent.
\newblock Representation learning: A review and new perspectives.
\newblock {\em IEEE transactions on pattern analysis and machine intelligence},
  35(8):1798--1828, 2013.

\bibitem{bousquet1705optimal}
O.~Bousquet, S.~Gelly, I.~Tolstikhin, C.-J. Simon-Gabriel, and B.~Schoelkopf.
\newblock From optimal transport to generative modeling: the vegan cookbook.
  2017.
\newblock {\em URL http://arxiv. org/abs/1705.07642}.

\bibitem{bowman2016generating}
S.~R. Bowman, L.~Vilnis, O.~Vinyals, A.~M. Dai, R.~Jozefowicz, and S.~Bengio.
\newblock Generating sentences from a continuous space.
\newblock {\em CoNLL 2016}, page~10, 2016.

\bibitem{burgess2018understanding}
C.~P. Burgess, I.~Higgins, A.~Pal, L.~Matthey, N.~Watters, G.~Desjardins, and
  A.~Lerchner.
\newblock Understanding disentangling beta-vae.
\newblock {\em arXiv preprint arXiv:1804.03599}, 2018.

\bibitem{chen2018isolating}
T.~Q. Chen, X.~Li, R.~B. Grosse, and D.~K. Duvenaud.
\newblock Isolating sources of disentanglement in variational autoencoders.
\newblock In {\em Advances in Neural Information Processing Systems}, pages
  2610--2620, 2018.

\bibitem{chen2016infogan}
X.~Chen, Y.~Duan, R.~Houthooft, J.~Schulman, I.~Sutskever, and P.~Abbeel.
\newblock Infogan: Interpretable representation learning by information
  maximizing generative adversarial nets.
\newblock In {\em Advances in neural information processing systems}, pages
  2172--2180, 2016.

\bibitem{chen2016variational}
X.~Chen, D.~P. Kingma, T.~Salimans, Y.~Duan, P.~Dhariwal, J.~Schulman,
  I.~Sutskever, and P.~Abbeel.
\newblock Variational lossy autoencoder.
\newblock {\em arXiv preprint arXiv:1611.02731}, 2016.

\bibitem{cybenko1989approximation}
G.~Cybenko.
\newblock Approximation by superpositions of a sigmoidal function.
\newblock {\em Mathematics of control, signals and systems}, 2(4):303--314,
  1989.

\bibitem{Ford2018lemmaproof}
N.~Ford and A.~Oliver.
\newblock Correcting a proof in the infogan paper, Mar. 2018.

\bibitem{Goodfellow-et-al-2016}
I.~Goodfellow, Y.~Bengio, and A.~Courville.
\newblock {\em Deep Learning}.
\newblock MIT Press, 2016.
\newblock \url{http://www.deeplearningbook.org}.

\bibitem{goodfellow2014generative}
I.~Goodfellow, J.~Pouget-Abadie, M.~Mirza, B.~Xu, D.~Warde-Farley, S.~Ozair,
  A.~Courville, and Y.~Bengio.
\newblock Generative adversarial nets.
\newblock In {\em Advances in neural information processing systems}, pages
  2672--2680, 2014.

\bibitem{gumbel1954statistical}
E.~Gumbel.
\newblock {\em Statistical theory of extreme values and some practical
  applications: a series of lectures}.
\newblock Applied mathematics series. U. S. Govt. Print. Office, 1954.

\bibitem{higgins2017beta}
I.~Higgins, L.~Matthey, A.~Pal, C.~Burgess, X.~Glorot, M.~Botvinick,
  S.~Mohamed, and A.~Lerchner.
\newblock beta-vae: Learning basic visual concepts with a constrained
  variational framework.
\newblock In {\em International Conference on Learning Representations},
  volume~3, 2017.

\bibitem{hoffman2016elbo}
M.~D. Hoffman and M.~J. Johnson.
\newblock Elbo surgery: yet another way to carve up the variational evidence
  lower bound.

\bibitem{huszar2017maximum}
F.~Husz{\'a}r.
\newblock Is maximum likelihood useful for representation learning, 2017.

\bibitem{jang2016categorical}
E.~Jang, S.~Gu, and B.~Poole.
\newblock Categorical reparameterization with gumbel-softmax.
\newblock {\em arXiv preprint arXiv:1611.01144}, 2016.

\bibitem{jordan1998learning}
M.~I. Jordan.
\newblock {\em Learning in graphical models}, volume~89.
\newblock Springer Science \& Business Media, 1998.

\bibitem{kim2018disentangling}
H.~Kim and A.~Mnih.
\newblock Disentangling by factorising.
\newblock In {\em International Conference on Machine Learning}, pages
  2654--2663, 2018.

\bibitem{kingma2017variational}
D.~P. Kingma.
\newblock Variational inference \& deep learning: A new synthesis.
\newblock 2017.

\bibitem{kingma2013auto}
D.~P. Kingma and M.~Welling.
\newblock Auto-encoding variational bayes.
\newblock {\em arXiv preprint arXiv:1312.6114}, 2013.

\bibitem{kumar2017variational}
A.~Kumar, P.~Sattigeri, and A.~Balakrishnan.
\newblock Variational inference of disentangled latent concepts from unlabeled
  observations.
\newblock {\em arXiv preprint arXiv:1711.00848}, 2017.

\bibitem{lake2016building}
B.~M. Lake, T.~D. Ullman, J.~B. Tenenbaum, and S.~J. Gershman.
\newblock Building machines that learn and think like people.
\newblock {\em arXiv preprint arXiv:1604.00289}, 2016.

\bibitem{maddison2016concrete}
C.~J. Maddison, A.~Mnih, and Y.~W. Teh.
\newblock The concrete distribution: A continuous relaxation of discrete random
  variables.
\newblock {\em arXiv preprint arXiv:1611.00712}, 2016.

\bibitem{makhzani2017pixelgan}
A.~Makhzani and B.~J. Frey.
\newblock Pixelgan autoencoders.
\newblock In {\em Advances in Neural Information Processing Systems}, pages
  1975--1985, 2017.

\bibitem{makhzani2015adversarial}
A.~Makhzani, J.~Shlens, N.~Jaitly, I.~Goodfellow, and B.~Frey.
\newblock Adversarial autoencoders.
\newblock {\em arXiv preprint arXiv:1511.05644}, 2015.

\bibitem{nguyen2010estimating}
X.~Nguyen, M.~J. Wainwright, and M.~I. Jordan.
\newblock Estimating divergence functionals and the likelihood ratio by convex
  risk minimization.
\newblock {\em IEEE Transactions on Information Theory}, 56(11):5847--5861,
  2010.

\bibitem{rezende2014stochastic}
D.~J. Rezende, S.~Mohamed, and D.~Wierstra.
\newblock Stochastic backpropagation and approximate inference in deep
  generative models.
\newblock In {\em International Conference on Machine Learning}, pages
  1278--1286, 2014.

\bibitem{rumelhart1985learning}
D.~E. Rumelhart, G.~E. Hinton, and R.~J. Williams.
\newblock Learning internal representations by error propagation.
\newblock Technical report, California Univ San Diego La Jolla Inst for
  Cognitive Science, 1985.

\bibitem{sugiyama2012density}
M.~Sugiyama, T.~Suzuki, and T.~Kanamori.
\newblock Density-ratio matching under the bregman divergence: a unified
  framework of density-ratio estimation.
\newblock {\em Annals of the Institute of Statistical Mathematics},
  64(5):1009--1044, 2012.

\bibitem{tolstikhin2017wasserstein}
I.~Tolstikhin, O.~Bousquet, S.~Gelly, and B.~Schoelkopf.
\newblock Wasserstein auto-encoders.
\newblock {\em arXiv preprint arXiv:1711.01558}, 2017.

\bibitem{Watanabe:1960:ITA:1661258.1661265}
S.~Watanabe.
\newblock Information theoretical analysis of multivariate correlation.
\newblock {\em IBM J. Res. Dev.}, 4(1):66--82, Jan. 1960.

\bibitem{zhao2017infovae}
S.~Zhao, J.~Song, and S.~Ermon.
\newblock Infovae: Information maximizing variational autoencoders.
\newblock {\em arXiv preprint arXiv:1706.02262}, 2017.

\end{thebibliography}


\acknowledgment[4]
    
    I am thankful to my advisor Professor Dae-Shik Kim for supporting and advising me during these past two years. I would like to express my deepest gratitude to my family for their love and support throughout the path. Without them I could not have done it. Finally, I thank my colleagues and KAIST professors for the support and knowledge that I have gained about various fields as well as the craft of research. I especially would like to thank Sun Mi Park and Chihye Han for being helpful and supportive when I needed it. 
    
    \textit{I dedicate this thesis to my mother. No words can express how grateful and sorry I am for all of the sacrifices that she has made and sufferings she has been through because of me. Thank you for everything and for your unconditional love that I do not deserve.}

\label{paperlastpagelabel}

\end{document}